\documentclass[journal]{IEEEtran}

\usepackage[utf8]{inputenc} 
\usepackage[T1]{fontenc}    
\usepackage{url}            
\usepackage{booktabs}       
\usepackage{amsfonts}       
\usepackage{nicefrac}       
\usepackage{microtype}      
\usepackage{xcolor}         

\usepackage{mathtools} 
\usepackage{booktabs} 
\usepackage{tikz} 
\usetikzlibrary{shadows}
\usetikzlibrary{bayesnet}
\usetikzlibrary{arrows.meta}
\usepackage[caption=false,font=normalsize,
   labelfont=sf,textfont=sf]{subfig}

\graphicspath{{figures/}}

\usepackage{amsmath}
\usepackage{amssymb}
\usepackage{amsthm}
\usepackage{bm}

\usepackage{hyperref}       
\newtheorem{theorem}{Theorem}
\newtheorem{corollary}[theorem]{Corollary}
\newtheorem{assumption}{Assumption}
\newtheorem{lemma}[theorem]{Lemma}
\newtheorem{remark}{Remark}

\newcommand{\calG}{\mathcal{G}}
\newcommand{\calP}{\mathcal{P}}

\newcommand\independent{\protect\mathpalette{\protect\independenT}{\perp}}
\def\independenT#1#2{\mathrel{\rlap{$#1#2$}\mkern2mu{#1#2}}}

\newcommand{\arrowast}{
\,
 \begin{tikzpicture}
    \draw [{Classical TikZ Rightarrow[length=0.9mm]}-{Rays[n=6]}] (0,0)--(0.35,0); 
\end{tikzpicture}
\,
}

\newcommand{\circleedgecircle}{
\,
 \begin{tikzpicture}
    \draw [{Circle[open]}-{Circle[open]}] (0,0)--(0.5,0); 
\end{tikzpicture}
\,
}

\newcommand{\astarrow}{
\,
 \begin{tikzpicture}
    \draw [{Rays[n=6]}-{Classical TikZ Rightarrow[length=0.9mm]}] (0,0)--(0.35,0); 
\end{tikzpicture}
\,
}

\newcommand{\circlearrow}{
\,
 \begin{tikzpicture}
    \draw [{Circle[open]}-{Classical TikZ Rightarrow[length=0.9mm]}] (0,0)--(0.4,0); 
\end{tikzpicture}
\,
}

\usepackage{cite}
\usepackage{algorithm}
\usepackage{algpseudocode}

\newcommand{\sscmt}[1]{\textcolor{orange}{[SS: #1]}}

\begin{document}

\title{Coordinated Multi-Neighborhood Learning on a Directed Acyclic Graph}

\author{Stephen Smith and Qing Zhou \thanks{The authors are with the Department of Statistics and Data Science, University of California, Los Angeles, CA (e-mail: stephensmith13424@gmail.com; zhou@stat.ucla.edu)}
}


\maketitle

\begin{abstract}
Learning the structure of causal directed acyclic graphs (DAGs) is useful in many areas of machine learning and artificial intelligence, with wide applications. 
However, in the high-dimensional setting, it is challenging to obtain good empirical and theoretical results without strong and often restrictive assumptions. Additionally, it is questionable whether all of the variables purported to be included in the network are observable. It is of interest then to restrict consideration to a subset of the variables for relevant and reliable inferences. In fact, researchers in various disciplines can usually select a set of target nodes in the network for causal discovery. This paper develops a new constraint-based method for estimating the local structure around multiple user-specified target nodes, enabling coordination in structure learning between neighborhoods. Our method facilitates causal discovery without learning the entire DAG structure. We establish consistency results for our algorithm with respect to the local neighborhood structure of the target nodes in the true graph. Experimental results on synthetic and real-world data show that our algorithm is more accurate in learning the neighborhood structures with much less computational cost than standard methods that estimate the entire DAG. An \textsf{R} package implementing our methods may be accessed at \url{https://github.com/stephenvsmith/CML}.
\end{abstract}


\maketitle

\section{Introduction}\label{sec:intro}

\IEEEPARstart{I}{n} recent years there has been ongoing development in structure learning algorithms of directed acyclic graphs (DAGs) for causal discovery and inference \cite{causalsurvey,causaldiscoverysurvey,kaddour2022causal}. However, these algorithms are limited by strong assumptions and sometimes intractable practical requirements for a large number of variables, which render them unreasonable or too restrictive for use in many applications. Similarly, in empirical settings rapid deterioration of the speed and accuracy of most algorithms has been observed as the size of the network increases even moderately \cite{guzhou2020}. 
Moreover, in fields such as genomics, researchers are often interested in causal discovery for only a few nodes in order to estimate their causal effects on other downstream nodes. This can be particularly challenging, especially since data sets frequently have many features and relatively few observations \cite{friedman2000,gene_reg_causal}. In situations such as these, global causal discovery methods suffer from the problems previously mentioned. For the purposes of this paper, we distinguish between global and local algorithms based on the proportion of network variables included in the algorithm and its output. Global algorithms aim to estimate the entire graph among all nodes, while local algorithms are limited to estimation of a proper and usually much smaller subgraph. It would be advantageous for a suitable algorithm to leverage local knowledge of the causal structure to adequately and reliably answer causal inquiries at a level at least commensurate with the performance of global algorithms. Referring again to the field of genomics, for example, it may be of interest to consider only a few target genes in a gene regulatory network and identify the causal effects of these on other genes of interest \cite{Spirtes2000_gene}. Therefore, instead of learning as much as we can of the global causal structure, we aim to learn only the identifiable structure on a subgraph sufficient for the estimation of causal effects of interest. This local approach reflects a priority to estimate the causal effects most relevant to researchers and pursues causal discovery accordingly. 

Let $X$ be a target node in a DAG $G$. If the parent set of $X$ is given, then one can calculate its causal effect on any other variable in the DAG by the so-called back-door adjustment \cite{pearl_causality_2009}. On this basis, we reason that local learning algorithms should be primarily focused on estimating the neighborhood structure around the specified target nodes, particularly to identify parents for the purpose of estimating causal effects. Indeed, it is this concept which serves as the basis for the IDA \cite{maathuisEstimatingHighdimensionalIntervention2009} and joint-IDA \cite{nandy_estimating_2017} algorithms in the estimation of causal effects, both of which assume target parent sets are given. However, current procedures for learning parent sets suffer from inefficiencies inherent in global causal discovery algorithms or from a failure to adequately distinguish between parents, children, and spouses in some Markov blanket (Mb) learning algorithms \cite{aliferis10a,aliferis2,gaoMarkovDiscovery}. Our method is designed to correct these shortcomings by maintaining the efficiency of a local learning approach while attempting to orient as many identifiable edges in the subgraph of the DAG as possible, thereby reducing the number of potential parent sets. This motivates us to pursue coordinated local learning, since by learning the structure of target neighborhoods simultaneously, we can orient more edges than we could by considering individual neighborhoods.

The advantages of pursuing coordinated local structure learning include relaxing global assumptions and a substantial reduction in computational complexity and runtime, which is demonstrated by our empirical results. Moreover, coordinated learning in principle permits information to pass from disjoint neighborhoods in the graph, thus providing greater specificity in causal effect estimation by defining the possible parent sets of target nodes with greater precision. To the best of our knowledge, the idea of coordinating local learning of multiple neighborhoods has not been thoroughly explored in the relevant literature.

\subsection{Relevant Work}

\noindent Researchers in a variety of fields working with high-dimensional data sets are seeking to reduce the size of vast feature sets while retaining predictive power or preserving causal relationships \cite{mbfeatureselection}. To address this demand, various local learning methods have been suggested. In particular, a substantial amount of work has been dedicated to estimating Mbs or the parent-child set, especially as such methods relate to the feature selection problem \cite{khanNovelFeatureSelection2022}. Many of these algorithms were developed for the purpose of improving predictive models, since a robust understanding of the causal mechanisms of the data generating process should improve the accuracy of predictions, particularly when the data undergoes an intervention of some kind \cite{mbMultiInterventions}. It has also been shown that the Mb is the theoretically optimal feature set for prediction given a faithful distribution \cite{tsamardinos03a,koller1996toward}. However, many Mb learning algorithms do not identify the parent set, thus limiting the potential effectiveness of causal effect estimation algorithms. Though this is a frequent shortcoming, there are some local methods which attempt to determine the parent, children, and spouse sets of the target node \cite{pellet_using_nodate,FangLocal}. 

While there are existing methods for estimating the neighborhood and the graphical structure around a single target node \cite{margaritis_mb,tsamardinos_large_scale,hiton03,niinimakiLocalStructureDiscovery,gaoMarkovDiscovery,borboudakisForwardBackwardSelectionEarly,GUO2022849}, there are, to our knowledge, no existing methods to coordinate structure learning around multiple neighborhoods of interest without learning the global structure. Applying existing methods to each neighborhood individually is sound for identifying members of each neighborhood. 
However, these methods are limited in their ability to orient edges within neighborhoods due to the inevitable loss of structural information when restricting consideration to a single neighborhood. Moreover, it would be particularly challenging to identify any topological ordering between the neighborhoods. Causal discovery in such a scenario is excessively restricted, and a method to coordinate structure learning over multiple neighborhoods is desired.

\subsection{Contribution}

\noindent In this paper, we develop a method to address this lacuna in the literature. The Coordinated Multi-Neighborhood Learning (CML) algorithm is designed to maximize causal structure learning in targeted neighborhoods with efficiency and scalability. We do not need to estimate the entire graphical structure, but instead we limit our attention to only the relevant subgraph over the target neighborhoods, which may be disjoint. Our algorithm first identifies target neighborhoods using existing Mb estimation methods. We then develop a two-stage constraint-based algorithm. The first stage is composed of two phases for skeleton recovery. The first phase constructs the skeleton of a maximal ancestral graph (MAG)~\cite{richardsonAncestralGraphMarkov2002,zhangCausalReasoningAncestral} over the union of target neighborhoods, maintaining ancestral relationships connecting the distinct neighborhoods. The second phase further prunes edges within each neighborhood by making use of additional conditional independence (CI) relations. The last stage involves applying a subset of the complete Fast Causal Inference (FCI) rules~\cite{zhang2008} to simultaneously orient edges in all neighborhoods. The contributions of this work are summarized below:

\begin{itemize}
    \item A novel framework for learning the local structures around multiple target nodes with coordinated edge orientation rules even for disjoint neighborhoods;
    \item Faster computation times and greater specificity in defining the parent sets of the target nodes for causal inference via the back-door adjustment;
    \item Theoretical guarantees for local structure learning, including high-dimensional consistency for Gaussian data.
\end{itemize}

The outline for the paper is as follows. In Section~\ref{sec:prelims}, we discuss basic graph terminology and review DAGs and ancestral graphs. In Section~\ref{sec:cml}, we will introduce CML and discuss its important features. Additionally, we provide theoretical results for our method, including its structure learning consistency in the large-sample limit and the high-dimensional setting. Then, we consider how the algorithm performs empirically on simulated data sets in Section~\ref{sec:empiricalanalysis} and on a real gene expression data set in Section~\ref{sec:generegulatory}. In Section~\ref{sec:conclusion}, we conclude by summarizing our findings and providing future research directions.

\section{Preliminaries}\label{sec:prelims}

\noindent Given a vector of random variables $X=(X_1,X_2,\ldots,X_p)$, a graph $G = (V,E)$ is composed of a set of nodes, or vertices, $V = \{V_1,\ldots,V_p\}$ and a set of edges $E \subseteq V \times V$. The nodes, which for convenience we also write as $V = [p]:=\{1,\ldots,p\}$, correspond to the elements of $X$.

If $(i,j) \in E$ but $(j,i) \notin E$, then there is a directed edge between nodes $i$ and $j$, denoted $i \rightarrow j$. If such an edge is present in $G$, then $i$ is a parent of $j$ and $j$ is a child of $i$. The parent set of node $i$ is denoted $pa_G(i)$, and the child set is denoted $ch_G(i)$. On the other hand, if both $(i,j),(j,i) \in E$, then $i - j$ is an undirected edge in $G$. The adjacency set of node $i$ in $G$, denoted $adj_G(i)$, is the set of all nodes connected to $i$ by any kind of edge. Each node $j \in adj_G(i)$ is said to be adjacent to $i$.

For endpoint nodes $i,j \in V$, a path $\pi$ in $G$ from $i$ to $j$ is a sequence of at least two adjacent nodes $\pi=\langle i=a_{1}, a_2,\ldots,a_{q}=j \rangle$ where all $a_k$'s are distinct. If each edge is oriented $a_k\to a_{k+1}$ on the path, then this is called a directed path from $i$ to $j$. If $G$ contains a path $\pi_v = \langle i,k,j \rangle$ where $i \rightarrow k$ and $k \leftarrow j$ and the path endpoints are non-adjacent, then these nodes form a $v$-structure denoted $(i,j,k)$. A spouse of $i$ in $G$ is a non-adjacent node which shares at least one child with $i$. The set of spouses of $i$ in $G$ is denoted $sp_G(i)$. For $v$-structure $(i,j,k)$ in $G$, $i \in sp_G(j)$ and $j \in sp_G(i)$.

If there is a directed path from $i$ to $j$ in $G$ as well as a directed path from $j$ to $i$, then $G$ contains a cycle. If $G$ does not contain any cycles and all the edges in $G$ are directed, then $G$ is called a Directed Acyclic Graph (DAG). Given a DAG $G$, the causal relations among $X$ are modeled via a structural equation model (SEM),
\begin{equation*}
X_i=f_i(X_{pa_G(i)},\varepsilon_i), \quad i=1,\ldots,p,
\label{eqn:SEM}
\end{equation*}
where $f_i$ are deterministic functions and $\varepsilon_i$ are independent background variables. This implies that the joint distribution $P(X)$ is Markov with respect to $G$, meaning the probability distribution admits the factorization
\begin{equation*}
    P(X_1,\ldots,X_p) = \prod_{i \in V} P(X_i \mid X_{pa_G(i)}),
\end{equation*}
according to $G$ \cite{pearl_causality_2009}. In addition, a probability distribution $P$ is said to be faithful to a DAG $G$ when there is a one-to-one correspondence between the CI relations of the distribution $P$ and the $d$-separations in the DAG. 

The Markov blanket (Mb) of node $i$ is the minimal set $mb_G(i)$ for which $X_i$ is rendered conditionally independent of the remaining nodes of the graph given $\{X_j : j \in mb_G(i)\}$ \cite{margaritis_mb,tsamardinos_large_scale}.
Faithfulness and causal sufficiency imply that the Mb is the union of a node's parents, children, and spouses, written as $mb_G(i) = pa_G(i) \cup ch_G(i) \cup sp_G(i)$ \cite{pearl1988probabilistic,aliferis10a}.
We define the set of first-order neighbors of a node $i$ to be its Mb, denoted $N_i^1$. We call $NB_i := N_i^1 \cup \{i\}$ the neighborhood of $i$. The second-order neighbor set of node $i$ is the union of the Mbs for each node in the first-order neighborhood, except for nodes in $NB_i$, denoted $N_i^2 := \cup_{j \in N_i^1} N_j^1 \setminus NB_i$.  For a set of nodes $T$, the union of their neighborhoods is denoted $NB_T := \cup_{t \in T} NB_t$. 

\subsection{Ancestral Graphs}

\noindent For the local learning problem in this work, we intend to learn the structure of local neighborhoods, coordinating the results such that causal information, as encoded in edge orientation, can, in principle, be passed from one neighborhood to another. Unlike many other constraint-based algorithms, ours treats some of the nodes as latent to facilitate inference of between-neighborhood edges and their orientations. For clarity, it must be acknowledged that none of the nodes are latent in the usual sense. That is, we have observed data for each of the these variables. However, many variables will not be included in the algorithm or its output.  
Thus, we prefer to say that these variables are {\em graphically latent}. That is, the nodes associated with these variables will not be included in the node set of the output graph, though they may be used during the execution of the algorithm. Specifically, only target nodes as well as first- and second-order neighbors (i.e., $NB_T \cup NB_T^2$, where $NB_T^2 := \cup_{t \in T} N_t^2$) will be involved in the CI tests of our local learning algorithm, while the remaining nodes will be excluded from our consideration. Moreover, second-order neighbors are excluded from the graphical output after being used for some CI queries.  

The local learning problem across multiple neighborhoods requires a different class of graphs to accommodate graphically latent variables while retaining the capacity for encoding causal information. For this purpose, we use ancestral graphs due to their facility for conveniently representing causal information on observed nodes in the presence of latent variables \cite{richardsonAncestralGraphMarkov2002,zhang2008}. 

Ancestral graphs are a class of mixed graphs. In this work, we only consider two kinds of edges in an ancestral graph $G$: directed ($\rightarrow$) and bi-directed ($\leftrightarrow$). We do not consider selection bias by assumption, thus removing the possibility of an undirected edge. The markings at each end of the edge are called marks or orientations. A directed edge contains a tail and an arrowhead, and a bi-directed edge contains two arrowheads. For edge $i \rightarrow j$ in $G$, the edge is said to be out of node $i$ and into node $j$ because the mark of the edge is a tail at $i$ and an arrowhead at $j$. Nodes $i$ and $j$ are siblings if $i \leftrightarrow j$, which implies that they share a common latent cause.

 For any pair of nodes $i,j \in V$, if $i = j$ or there is a directed path from $i$ to $j$, then we call $i$ an ancestor of $j$ and $j$ a descendant of $i$. The set of ancestors of $j$ in mixed graph $G$ is denoted $an_G(j)$. As with conventional graphs, if $i \rightarrow j$ and there is a directed path from $j$ to $i$ in $G$, then we can say $G$ contains a directed cycle. Also, if $i \leftrightarrow j$ and there is a directed path from $j$ to $i$ in $G$, then we can say $G$ contains an almost directed cycle. If a mixed graph $G$ does not contain any directed or almost directed cycles, then $G$ is ancestral \cite{zhang2008}.

A node $i$ is a collider on path $\pi$ if the two edges connecting to $i$ on $\pi$ are both into $i$ \cite{zhang2008}. All other nodes on the path are called non-colliders. Using these definitions, we can characterize the graphical criteria for CI relations in an ancestral graph. 
In a mixed graph, a path $\pi$ from $i$ to $j$ is $m$-connecting given a set of nodes $S\subseteq V \setminus \{i,j\}$ if 
(i) for every non-collider $k_{nc}$ on $\pi$, $k_{nc} \notin S$;
(ii) for every collider $k_c$ on $\pi$, there exists some $d\in S$ such that $k_c \in an_{G}(d)$ \cite{zhang2008}.
If there is no $m$-connecting path between $i$ and $j$ given $S$, then $i$ and $j$ are $m$-separated by $S$. Similar to $d$-separation for DAGs, $m$-separation in a mixed graph implies CI among observed variables via the global Markov property \cite{richardsonAncestralGraphMarkov2002}.  

A path $\pi$ from $i$ to $j$ is said to be inducing relative to set $S$ if every node on the path not in $S \cup \{i,j\}$ is a collider on the path and an ancestor of either $i$ or $j$ \cite{zhangCausalReasoningAncestral}. If $S$ is empty, we simply call $\pi$ an inducing path.
If there are no inducing paths between non-adjacent nodes in ancestral graph $G$, then $G$ is called a Maximal Ancestral Graph (MAG) \cite{richardsonAncestralGraphMarkov2002}. Accordingly, every pair of non-adjacent nodes in a MAG is $m$-separated by some subset of nodes. Thus, by definition, a MAG encodes CI relations among the observed nodes in a DAG with latent variables. Similar to DAGs, multiple MAGs may encode the same set of conditional independencies observable from data and form an equivalence class, represented by a partial ancestral graph (PAG).
A PAG has three possible kinds of marks with the addition of the circle ($\circ$) along with the tail and the arrowhead. Each circle mark corresponds to a variant mark and each non-circle mark is invariant in the equivalence class of a MAG \cite{zhangCausalReasoningAncestral}.

\subsection{FCI Algorithm}
\label{sec:fcirules}

\noindent The Fast Causal Inference (FCI) algorithm \cite{Spirtes2000,zhang2008} is a constraint-based algorithm designed for causal discovery in the presence of latent variables and selection bias not explicitly provided by the data set. It is similar to the PC algorithm, consisting of skeleton learning via CI tests and a set of edge orientation rules. After $v$-structure orientation ($\mathcal{R}_0$), the first three rules ($\mathcal{R}_1$-$\mathcal{R}_3$) are essentially the same as Meek's rules for learning the CPDAG of causal DAGs, and for that problem they are shown to be sound and complete \cite{meeksrules}. Rule $\mathcal{R}_4$ is particularly suited for MAGs containing bi-directed edges \cite{zhang2008}. This original set of rules was sound, but not complete. The rule set was augmented in \cite{zhang2008} with $\mathcal{R}_8$-$\mathcal{R}_{10}$, which are necessary to pick up all of the invariant tails such that the rule set is complete. Consequently, the output of the algorithm with a CI oracle (i.e., the population version) is the PAG representing the equivalence class of the true MAG. Note that we refrain from using rules $\mathcal{R}_5$-$\mathcal{R}_7$ in \cite{zhang2008} because these rules are only relevant where selection bias is present, and we assume the absence of selection bias in this work. The computational cost of the FCI algorithm can be very high for large graphs, and therefore, recent efforts have been made to improve its computational efficiency \cite{colomboLearningHighdimensionalDirected2012,chen2023}.

As we describe below, we will apply the FCI orientation rules in our algorithm to coordinate edge orientation among the neighborhoods of multiple target nodes. This will allow us to orient more edges and further reduce the uncertainty in causal parent identification than if we learned each neighborhood individually.

\section{CML Algorithm}
\label{sec:cml}

\noindent In many research scenarios, we can safely assume sufficient background knowledge to identify a set of target nodes $T$. These nodes will be designated by the user for the purpose of learning their local structures simultaneously, and the parent set of each node in particular. Recall that the parent set of a target node can be used as an adjustment set to identify the causal effect of this node on all other variables via the back-door adjustment. Our method may also partially identify the topological relationship among the target neighborhoods in the underlying DAG, though this global graph structure is never estimated.

Let $NB_T:=\cup_{t\in T}NB_t$ denote the union of the target neighborhoods, where $NB_t=N_t^1 \cup \{t\}$ is the neighborhood of node $t$. The basic idea of the CML algorithm is illustrated in Fig.~\ref{fig:illustration}, where $T = \{X_3,X_8\}$ is the set of target nodes. Fig.~\ref{fig:illustrationDAG} provides the true graphical structure of the targets and the first- and second-order neighbors, while there could be many other nodes outside of these neighborhoods. We want to learn the structure of the two neighborhoods $NB_3=\{X_1,X_2,X_3,X_4,X_5\}$ and $NB_8=\{X_7,X_8,X_9,X_{10}\}$. To coordinate structure learning, we make use of two paths $\pi_1=\langle 4,6,11,9\rangle$ and $\pi_2=\langle 2,12,9 \rangle$ between $NB_3$ and $NB_8$, of which the intermediate nodes are outside of the neighborhoods. Treating $L=\{X_6,X_{11},X_{12},X_{13}\}$ as latent variables, both $\pi_1$ and $\pi_2$ are inducing paths relative to $L$ and thus will be represented as edges in the MAG $\mathcal{M}$ over $NB_T$. The skeleton of $\mathcal{M}$ is shown in Fig.~\ref{fig:cmlskel}. Such between-neighborhood edges permit us to pass information between distinct neighborhoods through FCI orientation rules. On the other hand, by considering second-order neighbors of a target node, such as $X_{13}$, we can ensure that false positive edges within the same neighborhood, e.g. $(1,2)$, will be removed from the MAG $\mathcal{M}$, producing the estimated graph by CML shown in Fig.~\ref{fig:cml}.


\subsection{Algorithm Details}

\noindent For the population version of the CML algorithm, we assume perfect knowledge of the conditional independence information, or a CI oracle, which will be replaced with CI tests for the sample version. We assume that we are provided with (estimated) first- and second-order neighbors of each target node, $N_t^1$ and $N_t^2$, respectively, for $t\in T$. Estimation of neighbor sets is done by existing Mb learning algorithms; see the Supplementary Material for more details.

Algorithm~\ref{alg:CML} outlines the steps of the CML algorithm. Similar to the PC and FCI algorithms, the first stage of our algorithm is to recover the skeletal structure of the subgraph over the target neighborhoods $NB_T$. In the second stage, we apply inference rules to simultaneously orient edges in all the neighborhoods. 

\begin{algorithm}
\caption{Coordinated Multi-Neighborhood Learning \label{alg:CML}}
\begin{algorithmic}[1]
\State $O\gets NB_T$; $E\gets$ edge set of complete, undirected graph over $O$. \label{lst:line:inputs}
\For{$(i,j) \in E$} \label{lst:line:globalskel}
\State Search for separating set $S_{ij} \subseteq O \setminus \{i,j\}$ such that $X_i \independent X_j \mid S_{ij}$. \label{lst:line:sepsetsearch1}
\State If $S_{ij}$ is found, then update $E \leftarrow E \setminus \{(i,j),(j,i)\}$. \label{lst:line:updateedge1}
\EndFor \label{lst:line:endglobalskel}
\State $E_t \leftarrow \{(i,j) \in E: i,j \in NB_t \}$ for all $t\in T$. \label{lst:line:targetnbhd}
\For{$t \in T$} \label{lst:line:targetnbhdskel}
\For{$(i,j) \in E_t$}
\State Search for $S_{ij} \subseteq N^1_i \setminus \{j\}$ or $S_{ij}\subseteq N^1_j\setminus \{i\}$ such that $X_i \independent X_j \mid S_{ij}$. \label{lst:line:sepsetsearch2}
\State If $S_{ij}$ is found, then update $E \leftarrow E \setminus \{(i,j),(j,i)\}$.\label{lst:line:updateedge2}
\EndFor \label{lst:line:endcheckedges}
\EndFor \label{lst:line:endtargetnbhdskel}
\State Replace each edge with $\circleedgecircle$
\State Apply $\mathcal{R}_0$ of the FCI algorithm to identify $v$-structures based on $E$ and $\{S_{ij}\}$. \label{lst:line:vstruct}
\State Apply FCI rules $\mathcal{R}_1$ to $\mathcal{R}_4$ and $\mathcal{R}_8$ to $\mathcal{R}_{10}$ until none of them apply. \label{lst:line:fcirules}
\State Modify edge marks within each single neighborhood with rule $\mathcal{R}_N$. \label{lst:line:localrules}
\end{algorithmic}
\end{algorithm}

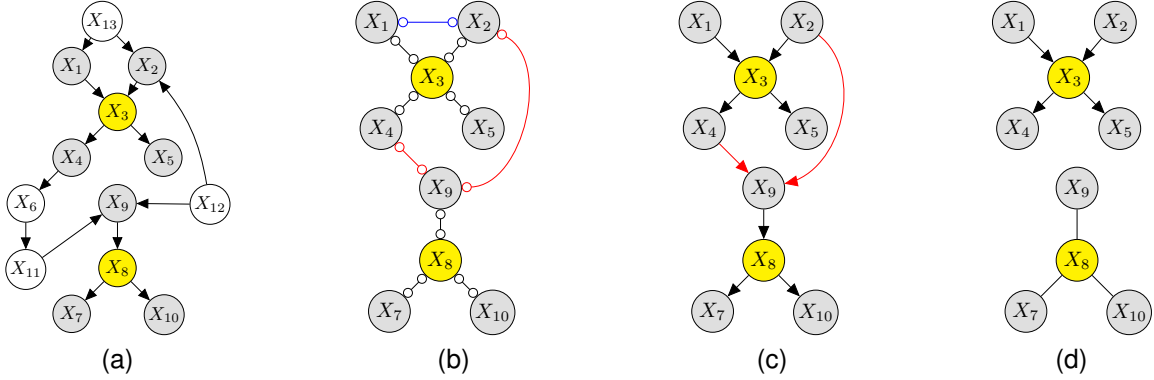
\begin{figure*}[!t]
    \centering
    \subfloat[]{
    \resizebox{!}{1.75in}{
    \begin{tikzpicture}[]
\node[obs] (X1) {$X_1$};
\node[obs, right=0.7 of X1] (X2) {$X_2$};
\node[obs,below right=0.5 of X1,fill=yellow] (X3) {$X_3$};
\node[obs,below left=0.5 of X3] (X4) {$X_4$};
\node[obs,below right=0.5 of X3] (X5) {$X_5$};

\node[latent,below left=0.5 of X4] (X6) {$X_6$};
\node[latent,above left=0.45 of X2] (X13) {$X_{13}$};

\node[obs,below=1 of X3] (X9) {$X_9$};
\node[obs,below=0.5 of X9,fill=yellow] (X8) {$X_8$};
\node[obs,below left=0.5 of X8] (X7) {$X_7$};
\node[obs,below right=0.5 of X8] (X10) {$X_{10}$};

\node[latent,below=0.5 of X6] (X11) {$X_{11}$};
\node[latent,below right=0.5 of X5] (X12) {$X_{12}$};

\edge[->] {X1,X2}{X3};
\edge[->] {X3}{X4,X5};
\edge[->] {X4}{X6};
\edge[->] {X6}{X11};
\edge[->] {X8}{X7};
\edge[->] {X8}{X10};
\edge[->] {X9}{X8};
\edge[->] {X11}{X9};
\draw[->] (X12) edge[bend right=15] node [left] {} (X2);
\edge[->] {X12}{X9};
\edge[->] {X13}{X1};
\edge[->] {X13}{X2};
\end{tikzpicture}
}
\label{fig:illustrationDAG}
    }
    \hfil
    \subfloat[]{
    \resizebox{!}{1.75in}{
    \begin{tikzpicture}[]
\node[obs] (X1) {$X_1$};
\node[obs, right=1 of X1] (X2) {$X_2$};
\node[obs,below right=0.6 of X1,fill=yellow] (X3) {$X_3$};
\node[obs,below left=0.5 of X3] (X4) {$X_4$};
\node[obs,below right=0.5 of X3] (X5) {$X_5$};

\node[obs,below right=0.7 of X4] (X9) {$X_9$};
\node[obs,below=0.5 of X9,fill=yellow] (X8) {$X_8$};
\node[obs,below left=0.5 of X8] (X7) {$X_7$};
\node[obs,below right=0.5 of X8] (X10) {$X_{10}$};

\draw[o-o] (X2) -- (X3);
\draw[o-o] (X1) -- (X3);
\edge[o-o] {X3}{X4,X5};
\draw[o-o,red] (X4) -- (X9);
\draw[o-o,red] (X2) edge[bend left=75] node [left] {} (X9);
\draw[o-o] (X8) -- (X7);

\edge[o-o] {X8}{X9};
\edge[o-o] {X8}{X10};
\edge[o-o,blue] {X1}{X2};
\end{tikzpicture}
}
\label{fig:cmlskel}
}
    \hfil
    \subfloat[]{
    \resizebox{!}{1.75in}{
    \begin{tikzpicture}[]
\node[obs] (X1) {$X_1$};
\node[obs, right=1 of X1] (X2) {$X_2$};
\node[obs,below right=0.6 of X1,fill=yellow] (X3) {$X_3$};
\node[obs,below left=0.5 of X3] (X4) {$X_4$};
\node[obs,below right=0.5 of X3] (X5) {$X_5$};

\node[obs,below right=0.7 of X4] (X9) {$X_9$};
\node[obs,below=0.5 of X9,fill=yellow] (X8) {$X_8$};
\node[obs,below left=0.5 of X8] (X7) {$X_7$};
\node[obs,below right=0.5 of X8] (X10) {$X_{10}$};

\draw[->] (X2) -- (X3);
\draw[->] (X1) -- (X3);
\edge[->] {X3}{X4,X5};
\draw[->,red] (X4) -- (X9);
\draw[<-,red] (X9) edge[bend right=65] node [left] {} (X2);
\draw[->] (X8) -- (X7);

\edge[<-] {X8}{X9};
\edge[->] {X8}{X10};

\end{tikzpicture}
}
\label{fig:cml}
}
    \hfil
    \subfloat[]{
    \resizebox{!}{1.75in}{
    \begin{tikzpicture}[]
\node[obs] (X1) {$X_1$};
\node[obs, right=1 of X1] (X2) {$X_2$};
\node[obs,below right=0.6 of X1,fill=yellow] (X3) {$X_3$};
\node[obs,below left=0.5 of X3] (X4) {$X_4$};
\node[obs,below right=0.5 of X3] (X5) {$X_5$};

\node[obs,below right=0.7 of X4] (X9) {$X_9$};
\node[obs,below=0.5 of X9,fill=yellow] (X8) {$X_8$};
\node[obs,below left=0.5 of X8] (X7) {$X_7$};
\node[obs,below right=0.5 of X8] (X10) {$X_{10}$};

\draw[->] (X2) -- (X3);
\draw[->] (X1) -- (X3);
\edge[->] {X3}{X4,X5};
\draw[-] (X7) -- (X8);

\edge[-] {X8}{X9};
\edge[-] {X8}{X10};
\end{tikzpicture}
}
\label{fig:singlenbhd}
    }
\caption{An illustration of the CML algorithm. (a) The neighborhoods of two target nodes. The highlighted nodes $\{X_3,X_8\}$ are the specified target nodes, the gray nodes are members of the Mb of one of the target nodes, and the white nodes are second-order neighbors. (b) The graph produced after the first phase of skeleton recovery. Edges in red are between-neighborhood edges and edges in black are within-neighborhood edges. The edge in blue will be removed during the second phase of skeleton recovery. (c) Output of the CML algorithm. (d) Output of the Single Neighborhood Learning (SNL) algorithm.}
\label{fig:illustration}
\end{figure*}

In the skeleton learning stage (lines~\ref{lst:line:inputs} to~\ref{lst:line:endtargetnbhdskel}), we begin with a complete graph over $NB_T$
and recursively delete edges based on the CI oracle, or CI tests in the sample version. 
However, in order to ensure that edges between the neighborhoods are properly maintained for coordinating orientations, this stage takes place in two successive phases. The first phase (lines~\ref{lst:line:globalskel} to~\ref{lst:line:endglobalskel}), the union skeleton recovery phase, is equivalent to the skeleton learning in the FCI algorithm, regarding $O=NB_T$ as the only observed variables. Thus, only subsets of $NB_T$ are possible separation sets. Edges between two different neighborhoods may be preserved due to lack of separation sets. Consider the two neighborhoods $NB_3=\{1,2,3,4,5\}$ and $NB_8=\{7,8,9,10\}$ in Fig.~\ref{fig:illustrationDAG}. No subsets of $NB_3\cup NB_8$ $m$-separate $X_4$ and $X_9$, and thus the edge $(4,9)$ will remain in the skeleton, and similarly the edge $(2,9)$ in Fig.~\ref{fig:cmlskel}. Note that the two edges correspond to inducing paths relative to $L=\{6,11,12,13\}$. Such between-neighborhoods edges will facilitate coordinated orientation in the latter stage of our algorithm.

After this phase, there may be extraneous edges present within each target neighborhood, such as the edge $(1,2)$ in Fig.~\ref{fig:cmlskel}. During the second phase (lines~\ref{lst:line:targetnbhd} to~\ref{lst:line:endtargetnbhdskel}), the local skeletons recovery phase, we narrow our focus to one target neighborhood at a time, each considered in succession. Now we make use of the second-order neighbors and search for separating sets in $N_i^1$ or $N_j^1$ for $i,j \in NB_t$ to further delete edges within the same neighborhood. Consider $N_1^1 = \{2,3,13\}$ and $N_2^1 = \{1,3,13\}$ in Fig.~\ref{fig:illustrationDAG} in order to remove the extraneous edge $(1,2)$. Though nodes 1 and 2 could not be separated by any subset of $NB_3$ in the first phase, they are separated by node $13$, a second-order neighbor of $X_{3}$, in the second phase. Therefore, the edge $(1,2)$ is removed in Fig.~\ref{fig:cml}.

\begin{remark}
A key difference between the two skeleton learning phases lies in the potential separating sets. In the first phase, only first-order neighbors $NB_T$ of the target nodes $T$ are considered, while the second-order neighbors are
included in the second phase for learning each individual neighborhood.
By construction every edge between two nodes not in the same neighborhood will remain connected during the second phase of skeleton recovery. In this way, learning across neighborhoods can be coordinated through these retained edges between the neighborhoods. On the other hand, by using second-order neighbors in the second phase, the local learning of individual neighborhoods is maximally informative.
\end{remark}

For the example in Fig.~\ref{fig:illustration} we discussed, we will obtain the skeleton of the graph in Fig.~\ref{fig:cml} after line~\ref{lst:line:endtargetnbhdskel} of our algorithm. If we only applied the first phase of skeleton learning, we would obtain 
the graph in Fig.~\ref{fig:cmlskel} with a false positive edge $(1,2)$ within one neighborhood. If we only applied the second phase, we would estimate the skeleton of the graph in Fig.~\ref{fig:singlenbhd} with no edges connecting the two neighborhoods, and thus could not coordinate their structure learning.

After the skeleton recovery stage, we identify $v$-structures using the stored separation sets (line~\ref{lst:line:vstruct}) and then apply the relevant FCI rules (line~\ref{lst:line:fcirules}) discussed in Section~\ref{sec:fcirules}. 
After line~\ref{lst:line:fcirules}, there may be four types of edges ($\leftrightarrow$, $\rightarrow$, $\circlearrow$, $\circleedgecircle$) in the estimated PAG. However, with knowledge of the first- and second-order neighbors, there should be no bidirected edges between two nodes in the same neighborhood. Consequently, we apply an additional set of rules ($\mathcal{R}_N$) to simplify the edge marks in a neighborhood.
\begin{list}{}{}
\item $\mathcal{R}_N$: For nodes $i$ and $j$ in the same target neighborhood, convert $i \circleedgecircle j$ to an undirected edge $i-j$, and convert $i \circlearrow j$ to a directed edge $i \to j$.
\end{list}
The soundness of $\mathcal{R}_N$ follows from the interpretation of the edge marks in MAGs and the fact that there are no bidirected edges between pairs of nodes in the same neighborhood. For an edge $i \circlearrow j$, the possible orientations are $i\to j$ or $i\leftrightarrow j$, but the latter (bidirected) is excluded within a neighborhood, and thus the orientation must be $i\to j$. This is because we assume knowledge of $N_t^1$ and $N_t^2$ and the absence of latent confounders for each neighborhood, which prevents there being any bidirected edge between the two nodes. A similar line of reasoning applies to changing $i \circleedgecircle j$ to $i - j$, since the latter still denotes uncertainty regarding causal direction while denying the possibility of a bidirected edge between nodes belonging to the same neighborhood.

For a finite sample, a bidirected edge could appear within a neighborhood in the CML output. 
In such a case, we also interpret it as an undirected edge to resolve this conflict in a practical way. This is similar to the situation of conflicting $v$-structures in DAG learning procedures. However, we keep these bidirected edges in the output in case there are unknown latent variables involved or missing neighbors in the estimated Mb.

The advantage of our method is illustrated in Fig.~\ref{fig:illustration}. Fig.~\ref{fig:cml} shows the learned graph by the CML algorithm, in which all edges in the two neighborhoods $NB_3$ and $NB_8$ have been oriented and the two between-neighborhood edges are shown in red. If we were to apply our algorithm to each of the two neighborhoods separately, then the output graphs would be the two shown in Fig.~\ref{fig:singlenbhd}. Because the subgraph over $NB_8$ contains no $v$-structure, none of the four edges would be oriented. This is in sharp contrast to the result in Fig.~\ref{fig:cml}, in which the $v$-structure $(2,4,9)$ leads to the orientation of all edges within $NB_8$ by compelling the remaining edges.


\begin{remark}
For each target node $t\in T$, the output graph from CML provides a set of parents $\widehat{pa}_t$ and possible parents $\widehat{pp}_t$ (nodes adjacent to $t$ by an undirected edge) in the neighborhood $NB_t$. Once we have these estimates of the target neighborhoods, we can use existing causal effect estimation procedures such as recursive regressions for causal effects (RRC) or modifying Cholesky decompositions (MCD) of the covariance matrix to infer causal effects of interest \cite{nandy_estimating_2017}; see the Supplementary Material for details. 
\label{rmk:causaleffectinputs}
\end{remark}

\subsection{Theoretical Analysis}
\label{sec:theoreticalanalysis}

\noindent To perform theoretical analysis of Algorithm~\ref{alg:CML}, we start with defining the ground-truth graph for the multiple neighborhood learning problem. Let $\calG=\calG(V)$ be a DAG over vertex set $V=[p]$, $T\subset V$ be a set of target nodes, and $N=NB_T$ be the union of the neighborhoods of $t\in T$. Then, $B:=N\times N-\cup_{t\in T} NB_t\times NB_t$ is the set of node pairs that do not belong to any common neighborhood, which is referred to as between-neighborhood pairs. Denote by $\calG_N$ the subgraph of $\calG$ over $N$. For each $(i,j)\in B$, if there is an inducing path between them relative to $L=V-N$ in $\calG$, add an edge between $i$ and $j$ to $\calG_N$ with the following orientation rules: (i) orient as $i\to j$ if $i\in\text{an}_\calG(j)$; (ii) orient as $j\to i$ if $j\in\text{an}_\calG(i)$; (iii) otherwise, orient as $i\leftrightarrow j$.
Denote the resulting graph as $\calG^*_N$. As an example, if $\calG$ is the DAG in Fig.~\ref{fig:illustrationDAG} with $T=\{3,8\}$, then $\calG^*_N$ is the graph in Fig.~\ref{fig:cml}.

\begin{assumption}\label{asp:inp}
In the DAG $\calG$, there is no inducing path relative to $L$ between any two nodes in the same neighborhood $NB_t$, $t\in T$, on which some intermediate node is in $N\setminus NB_t$.
\end{assumption}

\begin{lemma}\label{lm:mag}
Under Assumption~\ref{asp:inp}, the $\calG^*_N$ defined by the above procedure is a MAG.
\end{lemma}

Note that $\calG^*_N$ is in general a proper subgraph of, and thus sparser than, the MAG constructed by marginalizing $L$ from the DAG $\calG(V)$, which can have additional edges in a neighborhood due to inducing paths. 
For the DAG in Fig.~\ref{fig:illustrationDAG}, the MAG after marginalizing out $L=\{6,11,12,13\}$ is the graph
in Fig.~\ref{fig:cml} with an additional bidirected edge between $X_1$ and $X_2$.
However, in our construction, since $X_{13}$ is a second-order neighbor of target node $X_3$, this edge ($X_1-X_2$ in Fig.~\ref{fig:cmlskel}) would be deleted in the second skeleton recovery phase, resulting in the graph of Fig.~\ref{fig:cml}.

Due to Lemma~\ref{lm:mag}, the Markov equivalence class of $\calG^*_N$ is represented by a PAG, denoted $[\calG^*_N]$.
Given the CI oracle, which also can be used to perfectly recover the neighbor set, or Mb, of any node, we have the following result for the population version of our algorithm: 
\begin{theorem}\label{thm:pop}
Suppose the joint distribution $P(X_1,\ldots,X_p)$ is faithful to $\calG$ and Assumption~\ref{asp:inp} holds. Given the CI oracle, the graph constructed by Algorithm~\ref{alg:CML} up to the completion of line~\ref{lst:line:fcirules} is the PAG $[\calG^*_N]$.
\end{theorem}

It is straightforward to establish structure learning consistency for the sample version of our algorithm when the CI oracle is replaced by pointwise consistent CI tests. Let $\mathcal{T}_n$ be a test for $H_0: \theta\in\Theta_0$ versus the alternative $H_a: \theta\in \Theta_1$,
and denote by $\alpha(\mathcal{T}_n\mid\theta_0)$, $\theta_0\in\Theta_0$, and $\beta(\mathcal{T}_n\mid\theta_1)$, $\theta_1\in\Theta_1$,
its type I and type II error probabilities, respectively. 
We say $\mathcal{T}_n$ is pointwise consistent if
\begin{align*}
\lim_{n\to\infty}\alpha(\mathcal{T}_n\mid\theta_0)=0 \quad\text{and}\quad \lim_{n\to\infty}\beta(\mathcal{T}_n\mid\theta_1)=0
\end{align*}
for all $\theta_0\in\Theta_0$ and $\theta_1\in\Theta_1$.
Denote by $\widehat G_n$ the graph constructed by Algorithm~\ref{alg:CML} up to the completion of line~\ref{lst:line:fcirules} given a sample of size $n$ from $P(X_1,\ldots,X_p)$.
\begin{theorem}\label{thm:consistency}
Suppose the joint distribution $P(X_1,\ldots,X_p)$ is faithful to $\calG$ and Assumption~\ref{asp:inp} holds. Perform all CI checks in Algorithm~\ref{alg:CML} using pointwise consistent CI tests. Then, $P(\widehat G_n=[\calG^*_N])\to 1$ as $n\to \infty$.
\end{theorem}

With these results, we have shown that our algorithm is sound and complete. When the sample version of the algorithm is used with consistent CI tests, we have consistency with respect to the Markov equivalence class of the ground truth $\calG_N^*$. These results will hold for any data distribution as long as a consistent CI test is used in the CML algorithm. 

To illustrate Assumption~\ref{asp:inp}, consider the DAG in Fig.~\ref{fig:illustrationDAG}. The path $\pi=\langle 1,13,2 \rangle$ is an inducing path relative to $L=\{6,11,12,13\}$ between two nodes in the same target neighborhood, $NB_{3}$. However, the only intermediate node $X_{13}$ on the path is in $L$ and not $NB_8$, the assumption holds. This allows us to remove the edge between nodes 1 and 2 using the second-order neighbor node 13. For a different path, consider $\pi=\langle 4,6,11,9,12,2 \rangle$. Here, node 9 is an intermediate node on the path belonging to a different target neighborhood, $NB_8$, than the endpoints which both belong to $NB_3$. However, since node 9 is not an ancestor of either node 4 or node 2, this is not an inducing path. Thus, nodes 2 and 4 may be properly separated by node 3. Using this reasoning, it is evident that neither of the two paths violates Assumption~\ref{asp:inp}. Indeed, this is a weak assumption which will only be violated in rare, pathological cases.


Let $f(k)$ and $g(k)$ be the respective computational complexities of the FCI and the PC algorithms on a $k$-node problem. Then the computational complexity of Algorithm~\ref{alg:CML} is bounded by $f(|N|)+\sum_{t\in T} g(|NB_t|)$. 
In the worst case, the computational complexity of the FCI and the PC algorithm is exponential in the number of nodes \cite{Spirtes2000,pmlr-vR3-spirtes01a}. Assuming sparsity in the underlying DAG, the PC algorithm improves to  polynomial complexity \cite{kalischEstimatingHighDimensionalDirected}. Clearly, our local algorithm will achieve substantial computational savings when $|N|\ll p$ compared to applying the PC algorithm on all the $p$ nodes. This will be further demonstrated with our numerical comparisons.

\subsection{Consistency for the Gaussian Case}

\noindent As a special case, we develop structure learning consistency results for the CML algorithm assuming Gaussian linear structural equation model associated with the DAG $\calG$,  
\begin{equation}\label{eq:Gaulsem}
X_j = \sum_{i \in pa_{\calG}(j)} \beta_{ij}X_i + \varepsilon_j, \quad j \in [p],  
\end{equation}
where $\varepsilon_j$ are independent Gaussian errors.

For any pair of nodes $i,j$ and potential separating set $S$ (which may be empty), we can test whether or not  $X_i$ and $X_j$ are conditionally independent given $X_S$ by carrying out the hypothesis test $H_0: \rho_{i,j\mid S} = 0$, where $\rho_{i,j\mid S}$ is the partial correlation between $X_i$ and $X_j$ given $X_S$.
To test this, we apply Fisher's z-transformation to obtain the test statistic 
\[Z(i,j;S) = \frac{1}{2}\log\left(\frac{1+\hat{\rho}_{i,j\mid S}}{1-\hat{\rho}_{i,j\mid S}}\right),\]
where $\hat{\rho}_{i,j\mid S}$ is the sample correlation. Then, a p-value is calculated using the approximation $\sqrt{n-|S|-3}[Z(i,j;S)] \sim N(0,1)$ when $n$ is large \cite{kalischEstimatingHighDimensionalDirected}.

\begin{corollary}\label{cor:Gauconsistency}
Suppose the joint distribution of $X_1,\ldots,X_p$ is given by the Gaussian linear SEM~\eqref{eq:Gaulsem} and is faithful to $\calG$. Further assume that Assumption~\ref{asp:inp} holds. Perform all CI checks in Algorithm~\ref{alg:CML} by the Fisher's z-test on partial correlations with significance level $\alpha_n$. Then, there exists $\alpha_n\to 0$, such that 
$P(\widehat G_n=[\calG^*_N])\to 1$ as $n\to \infty$.
\end{corollary}

In addition to the classical asymptotic setting, we will also prove the consistency of our method in the high-dimensional setting, where the DAG $\calG=\calG_n$ with dimension $p=p_n$ now can grow with the sample size $n$. Denote by $N_n$ the union of the target neighborhoods and let $N_n^2$ be the set of second-order neighbors of the target nodes. The union of the target neighborhoods and the second-order neighbors is denoted $N_n^{1,2} = N_n \cup N_{n}^2$. Let $D_n = |N_n^{1,2}|$ be the size of the set $N_n^{1,2}$. 
In the high-dimensional setting, it is common to set a maximum size $m_n$ for the candidate separating sets in constraint-based learning. Let $\widehat{G}_n(m_n)$ be the graph generated from Algorithm~\ref{alg:CML} after completion of line~\ref{lst:line:fcirules} using significance level $\alpha_n$ and maximum separating set size $m_n$.

\begin{assumption}\label{asp:high-dim} We make the following assumptions for the high-dimensional Gaussian case. 
    \begin{itemize}
        \item[(A1)] The distribution $P(X_1,\ldots,X_p)$ is defined by the Gaussian linear SEM~\eqref{eq:Gaulsem} and is faithful to $\calG_n$.
        \item[(A2)] The neighborhood union size $D_n = O(n^a)$ for some $a \geq 0$.  
        \item[(A3)] Let $\text{m-Sep}_{\calG^\ast_{N_n}}(i,j)$ be the smallest set that $m$-separates non-adjacent nodes $i$ and $j$ in the MAG $\calG^\ast_{N_n}$, $\ell_n = \sup_{i,j}|\text{m-Sep}_{\calG^\ast_{N_n}}(i,j)|$, 
        and $\nu_n = \sup\{|adj_{\calG_n}(i)|: i \in N_n\}$. We assume $q_n := \max(\ell_n,\nu_n)= O(n^{1-b})$ for some $0 < b < 1$.
        \item[(A4)] For any pair of distinct nodes $i,j \in N_n$ and subset $S \subseteq N_n^{1,2} \setminus \{i,j\}$ with size $|S|\leq m_n$, let $\rho_{i,j\mid S}$ be the partial correlation between ${X}_i$ and ${X}_j$ given $\{{X}_k : k \in S\}$. 
        We assume
        \[
        \sup_{i,j,S} |\rho_{i,j \mid S}| \leq M < 1
        \]
        for some constant $M$ and
        \[
        \inf\{|\rho_{i,j \mid S}|: \rho_{i,j\mid S} \neq 0 \}  \geq c_n,
        \]
        where $c_n^{-1} = O(n^d)$ for some $0 < d < b/2$ and $b\in(0,1)$ is defined in (A3). 
        
        \item[(A5)] As $n \to \infty$, the estimated output from the Mb recovery algorithm, $\widehat{N}_n^{1,2}$, fulfills
        \begin{align*}
        P\left[ \widehat{N}_n^{1,2} = N_n^{1,2} \right] &\geq 1 - O(|N_n|\exp(-C_{Mb}n^{1-\delta})),
        \end{align*}
        where $0 < \delta < 1$ and $0 < C_{Mb} < \infty$. 
    \end{itemize}
\end{assumption}

\begin{theorem}\label{thm:pagconsistency} 
Suppose Assumption~\ref{asp:inp} and (A1)-(A5) of Assumption~\ref{asp:high-dim} all hold. Let $\calP_{N_n}^\ast = [\calG_{N_n}^\ast]$ be the PAG representing the equivalence class of the MAG $\calG_{N_n}^\ast$ defined above. 
If $q_n\leq m_n=O(q_n)$, then there exists some significance level $\alpha_n \to 0$ such that  
\begin{align*}
    &P\Big[ \widehat{G}_n(m_n) = \calP_{N_n}^\ast \Big] = \\
    &\quad\quad 1 - O\left( \exp(-C n^{1-2d})\right) - O\left(\exp(-C_{M}n^{1-\delta})\right)\to 1
\end{align*}

as $n \to \infty$ for some positive constants $C$ and $C_M$. 
\end{theorem}

Assumption (A3) is a sparsity condition over the target neighborhoods and the second-order neighbors to bound the complexity of both skeleton recovery phases. Assumption (A4) is a local version of strong faithfulness, as it is restricted to $N_n^{1,2}$ instead of the entire graph $G_n$. Both (A3) and (A4) significantly relax the sparsity and strong faithfulness conditions for corresponding global algorithms such as the PC. Finally, in (A5) we assume the Mb recovery algorithm accurately recovers the first- and second-order neighbors of the target nodes after $|N_n|$ runs: once for each target node and each first-order neighbor. It is drawn from a similar statement in Section 2 of \cite{meinshausen} for neighborhood estimation using the Lasso. Note that we make no specific assumptions for the number of nodes $p_n$ of the entire DAG, which may be much larger than $D_n$. However, the order of $p_n$ is relevant implicitly for the Mb accuracy assumption in (A5).

\section{Simulation Experiment Results}
\label{sec:empiricalanalysis}

\noindent The algorithm we have discussed is flexible such that it can be used for any type of data. 
For our simulations, however, we assume a linear SEM with independent Gaussian errors such as we find in equation~\eqref{eq:Gaulsem}. For our analysis, we simulate data from 13 networks provided by the \textbf{bnlearn} network repository \cite{jstatsoft09}. The SEM coefficients are drawn from the Unif$(0.4,0.75)$ distribution, multiplied by a sign term with equal probability of being positive or negative. We also generate error terms from $N(0,\sigma_j^2)$ with the standard deviation $\sigma_j$ drawn from Unif$(0.1,0.5)$. With these parameters, we randomly generate data sets of size $n \in \{500, 1000, 10000\}$. To estimate Mbs and obtain first- and second-order neighbors for each target, we use the MMPC algorithm from the \textbf{MXM} package \cite{mmpc}. For our analysis, we use an augmented MMPC algorithm to estimate the first- and second-order neighbors for each target node, applying an additional CI test to recover spouses belonging to the first-order neighbor sets; see the Supplementary Material for more details. 
For both the Mb estimation and skeleton recovery steps, we use significance levels $(\alpha_{Mb},\alpha_{skel}) \in \{0.01,0.05,0.1\}^2$, where $\alpha_{Mb}$ is used for the Mb recovery algorithm and $\alpha_{skel}$ is the significance level for the CI tests in CML skeleton recovery. We use the sample partial correlation as our statistic for CI testing. The elements described above together form a unique simulation setting, which is composed of a network, data set size, Mb estimation algorithm, and significance level pair. For each setting, three data sets are generated with unique randomly drawn coefficients and error variances.

For each network, we randomly select a set of target sets of varying cardinality ranging from two to four nodes from the entire vertex set. These target sets are used in each simulation setting for the network. After generating each data set, we ran the global PC algorithm one time, using the implementation from the \textbf{pcalg} library \cite{pcalg}. For additional comparison on each target set, we also ran the Single Neighborhood Learning (SNL) algorithm along with CML. One may consider SNL an augmented version of the global Grow-Shrink \cite{margaritis_mb} applied to the subgraph over the neighborhoods of multiple targets, and for our purposes represents the extension of any single Mb algorithm to the multi-neighborhood problem. The SNL applies the PC algorithm to each neighborhood individually, without any coordination between disjoint neighborhoods. Fig.~\ref{fig:singlenbhd} provides the output of SNL for the example in Fig.~\ref{fig:illustrationDAG}; see the Supplementary Material for more details. 
Both SNL and CML are applied on all data sets, and once for each target set. We also filter our results by the number of nodes in target neighborhoods and the number of estimated edges to ensure that we are only considering relatively sparse, smaller neighborhoods. In the analysis below, we only use target sets such that the total number of nodes under consideration, according to the subgraph of the ground truth CPDAG, is in the set $[8,20]$, and the number of edges is in the set $[3,20]$. We further filter the simulations to only include $\alpha_{Mb}=0.01$ since this was the optimal choice for MMPC across different data sets according to F1 score with respect to the true Mb in the underlying DAG.

Let $G'$ be the CPDAG of DAG $G$, and $G'_{NB_T}$ be the subgraph of $G'$ over $NB_T$ for a given target set $T$. For each estimated setting and target set, we compare the estimated graph to $G'_{NB_T}$, since this represents the maximal amount of information which is not underdetermined and in principle recoverable by a global structure learning algorithm, say the PC algorithm, to which we are comparing our local method. We consider this to be the ground truth against which we measure the performance of our algorithms. In order to compare our local methods with a suitable global competitor, we also measure the subgraph of the PC algorithm output over $NB_T$ against the ground truth.

\subsection{Overall Accuracy and Runtime}

\noindent The overall F1 score measures how well each estimated graph precisely conforms to the ground truth graph. The overall F1 score measures how well the edge set of the estimated graph precisely conforms to that of the ground truth graph, and is given by $\text{F1} = \frac{2TP}{2TP + FP + FN + IO}$, where $TP$, $FP$, and $FN$ are the number of true positives, false positives, and false negatives, respectively. The last term, $IO$, denotes the number of edges in the estimated graph which have incorrect orientation with respect to the ground truth graph. These errors are distinct from false positives and false negatives because the adjacency relations are still correct. In Fig.~\ref{fig:overallf1}, we compare our method to the global PC algorithm as well as to SNL applied to each target node. The overall F1 score measures the F1 score for all the edges in the estimated graph compared to $G'_{NB_T}$. In terms of overall performance, the CML algorithm is superior to the other methods in both smaller (top panels, $p<100$) and larger networks (lower panels, $p\geq 100$). For example, the median F1 score of the CML algorithm is 159\% higher than that of the global PC algorithm, using significance level of 0.01. The same conclusion is drawn if using the Structural Hamming Distance (SHD) as the measure of structure learning accuracy. See the Supplementary Material for detailed comparisons of the algorithms with respect to SHD.

\begin{figure}[!t]
\centering
\includegraphics[width=3in]{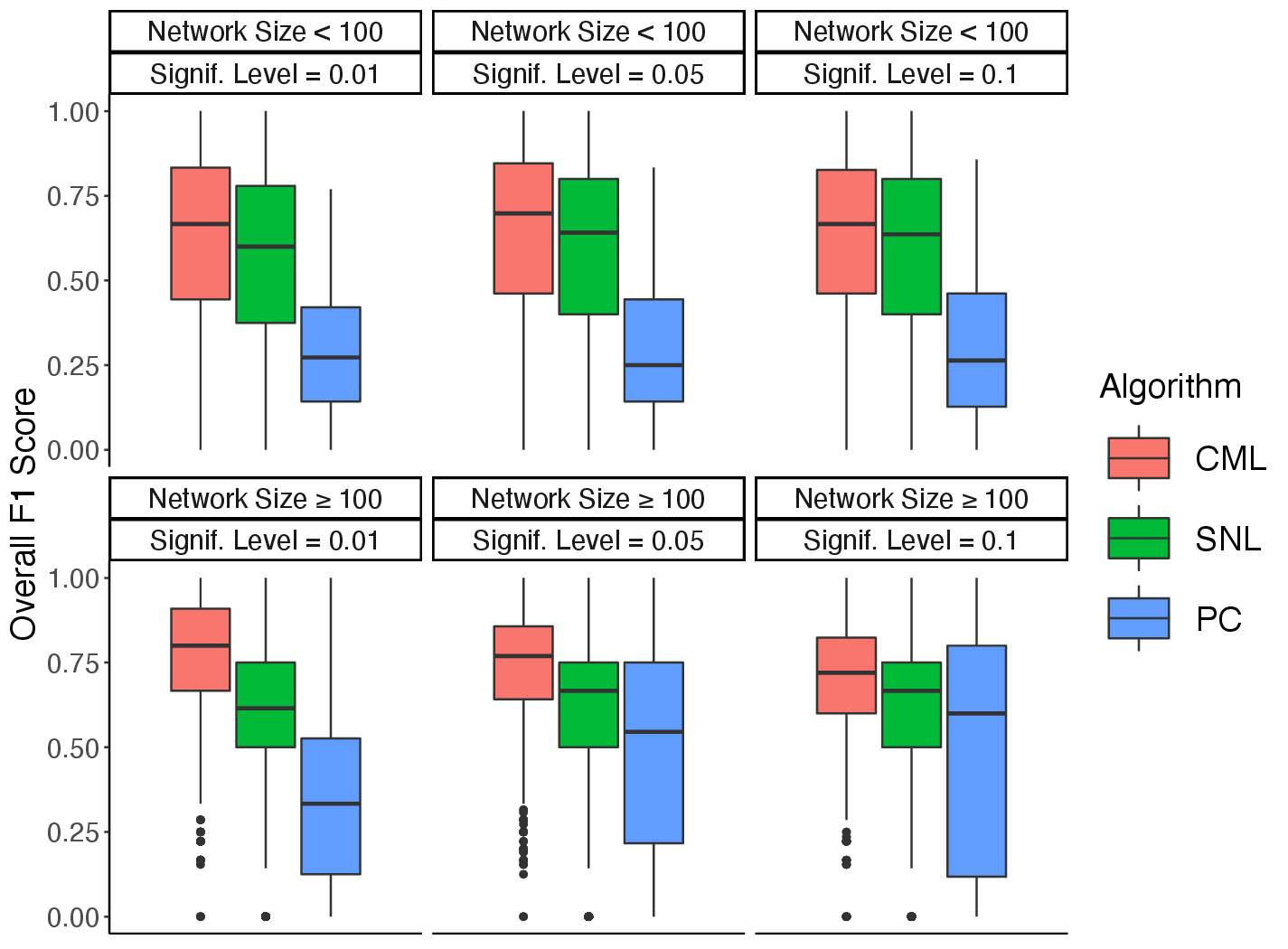}
\caption{Comparisons between the global and local algorithms with respect to accuracy by providing the distributions of F1 scores for different combinations of network size category and skeleton CI test significance levels.}
\label{fig:overallf1}
\end{figure}

\begin{figure*}[!t]
    \centering
    \subfloat[]{
    \includegraphics[width=3in]{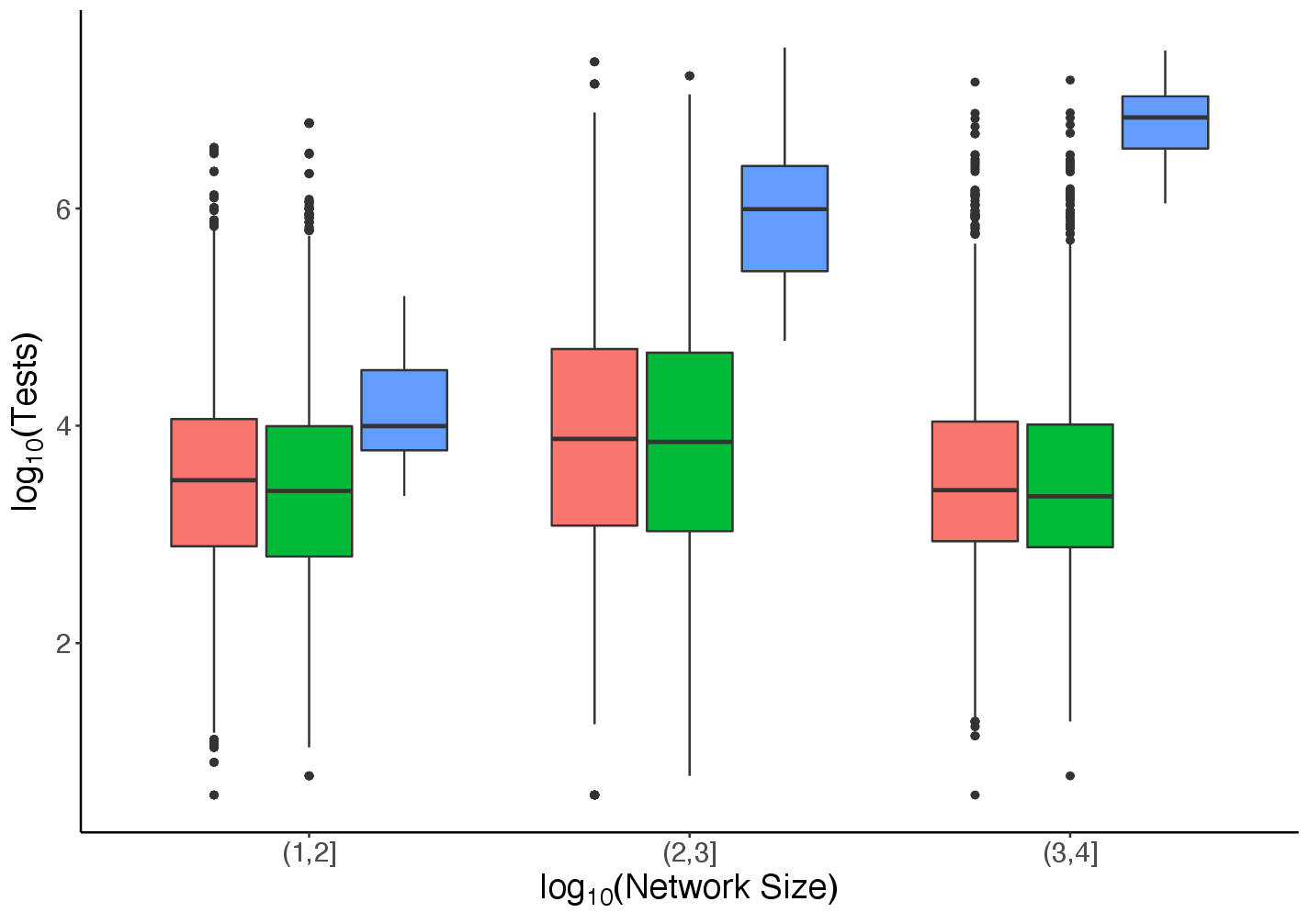}%
    \label{fig:testplot}
    }
    \hfil
    \subfloat[]{
    \includegraphics[width=3in]{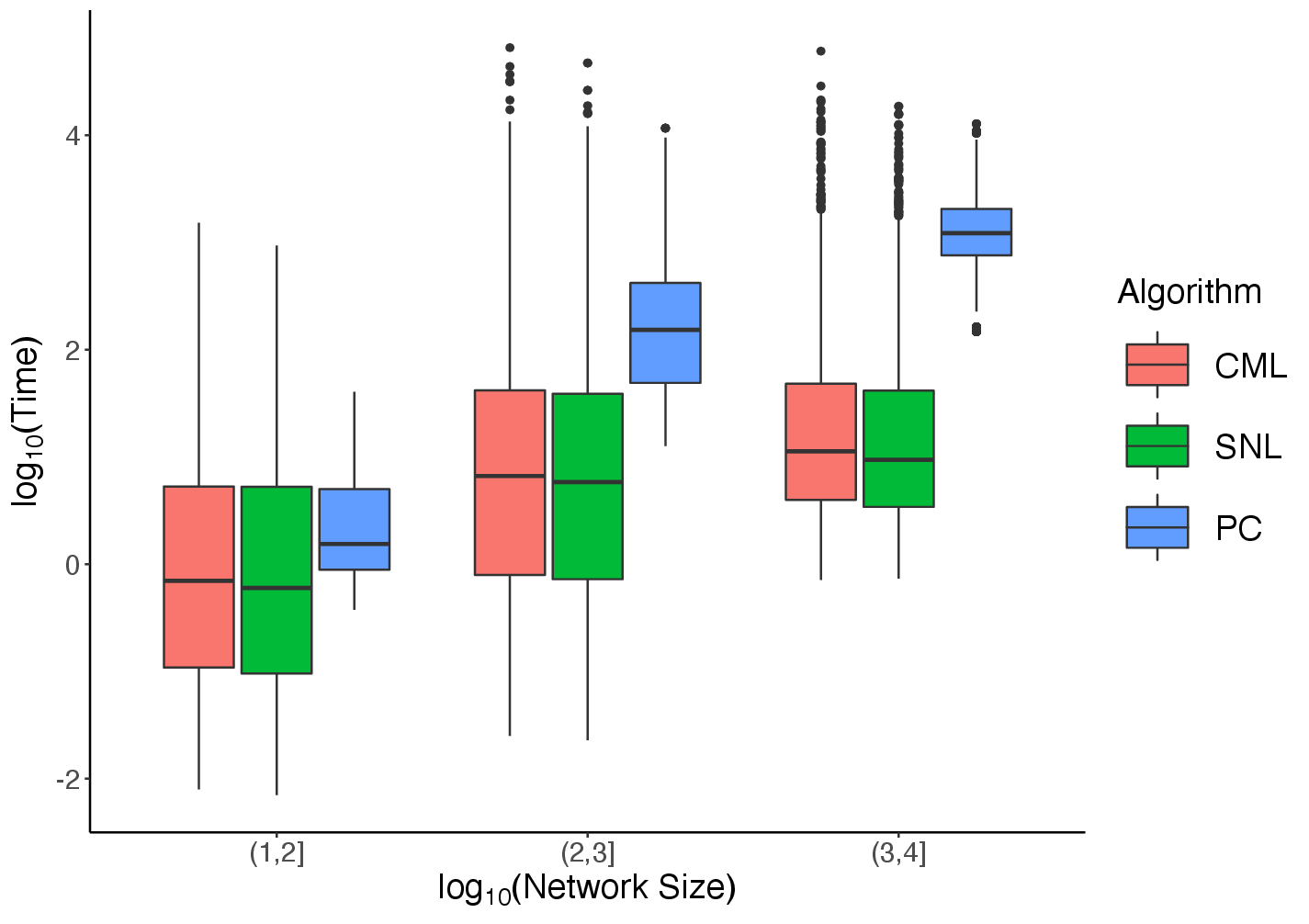}%
    \label{fig:timeplot}
    }
    \caption{Comparisons between the global and local algorithms with respect to complexity. (a) The distributions of the number of CI tests used by each algorithm for different network sizes on a log scale. (b) The distributions of runtime 
for different combinations of network size category and skeleton CI test significance levels on a log scale.}
\label{fig:complexityandruntime}
\end{figure*}

Runtime improvement is perhaps the most significant contribution of local methods to causal discovery. Since the computation of all three algorithms is dominated by CI tests, we use the total number of tests performed as a metric for computational cost. As demonstrated in Fig.~\ref{fig:testplot}, the number of CI tests executed sharply decreases for the local algorithms when compared to the PC algorithm. For networks with more than 100 nodes, we observe a reduction of nearly two orders of magnitude in the median number of CI tests. A similar pattern of improvement is observed in the actual runtime of the complete algorithms, as shown in Fig.~\ref{fig:timeplot}.

We must note, however, that the number of CI tests reported for the local algorithms does not include the number of tests required to estimate the Mb, since we were unable to obtain these values. This consideration does not substantially change our conclusions. Assume that the complexity of a Mb recovery algorithm is $O(|N|p)$, the complexity for the Grow-Shrink algorithm, where $N$ is the union of the neighborhood sets and $p$ is the size of the network. Assuming that $|N|$ is bounded by a constant, visual inspection allows us to safely conclude that the difference in the number of CI tests between the local and global algorithms will not be substantially altered by the addition of Mb recovery CI tests, as the distribution of tests for the PC algorithm in Fig.~\ref{fig:testplot} is clearly much greater than $\log p$. Also, the runtime statistics in Fig.~\ref{fig:timeplot} do include the Mb recovery time, and we still observe a similar pattern of improvement. However, the PC algorithm is slowed down to some degree by manual calculation of the number of tests, though this should lead to only a relatively small increase in the PC runtime. 

\subsection{Parent Recovery}

\begin{figure*}[!t]
    \centering
    \subfloat[]{
    \includegraphics[width=3in]{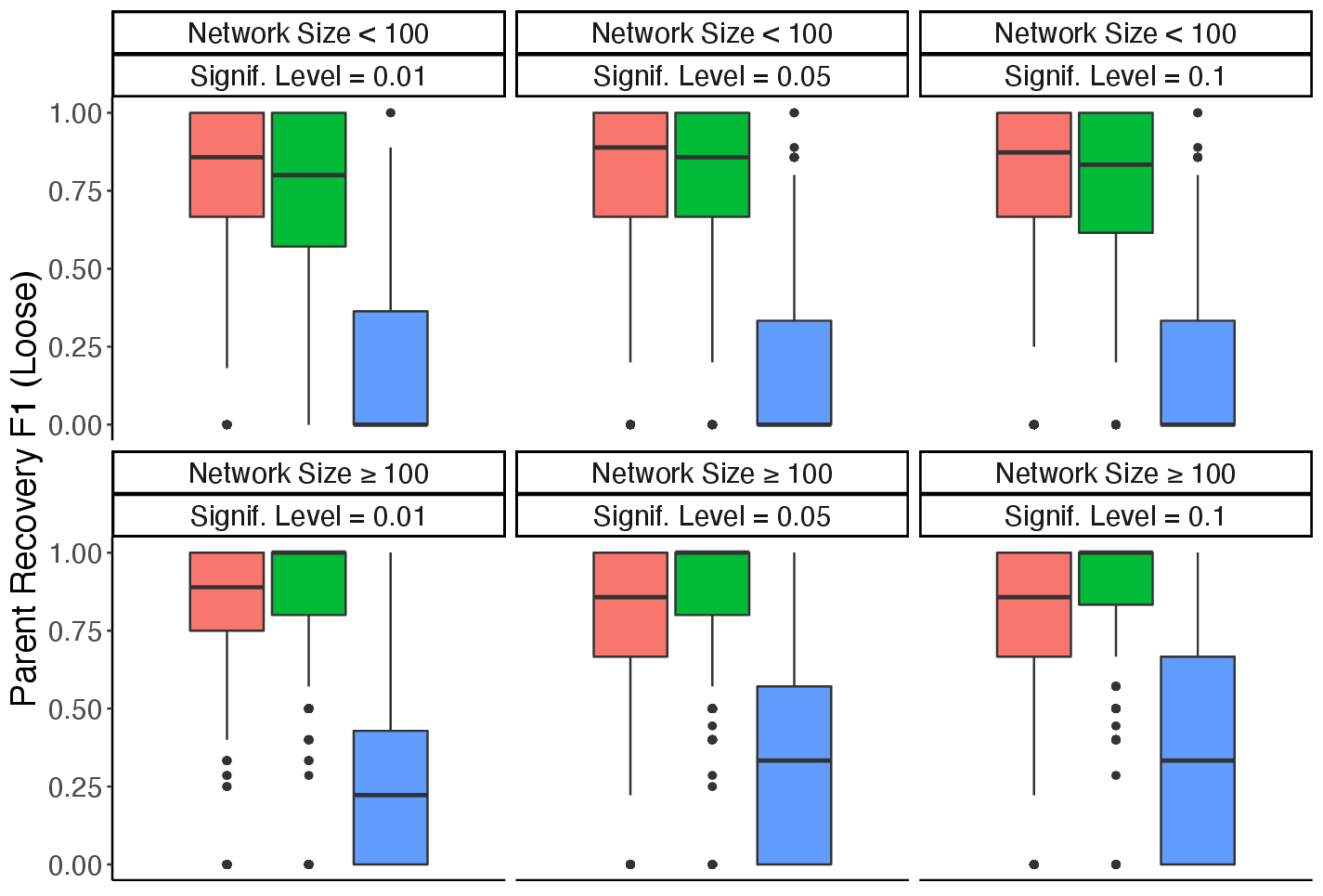}%
    \label{fig:praf1loose}
    }
    \hfil
    \subfloat[]{
    \includegraphics[width=3in]{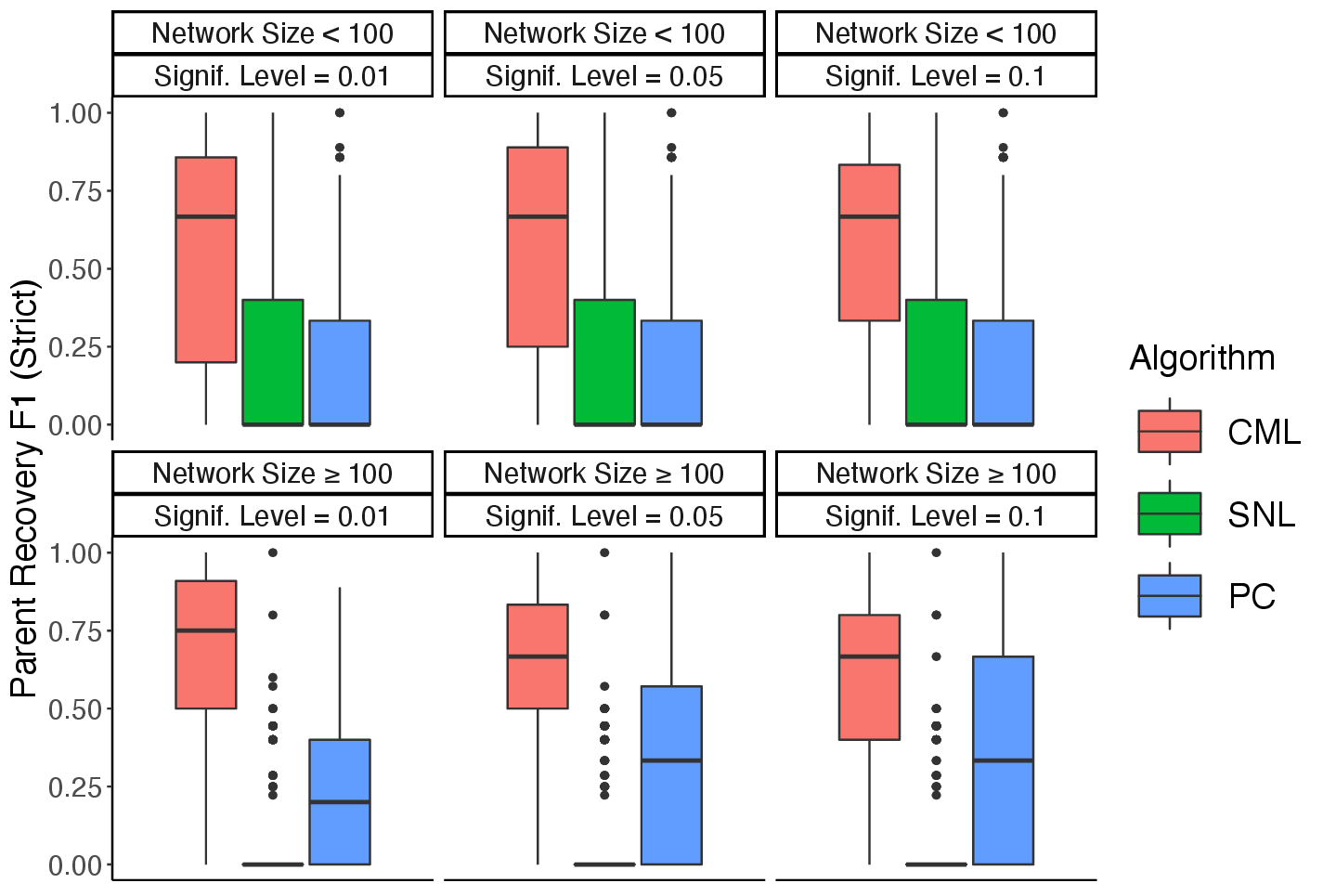}%
    \label{fig:praf1strict}
    }
    \caption{The distributions of the parent recovery accuracy F1 scores  for different network size and significance level combinations. (a) The loose F1 score; (b) The strict F1 score.}
\label{fig:parentrecovery}
\end{figure*}

\noindent One of the primary goals of this work is to identify as many parents of a target node as possible, since this would reduce the uncertainty in parent set adjustment for estimating causal effects (cf. Remark~\ref{rmk:causaleffectinputs}). Therefore, beyond pure conformity to the ground truth, it is of special interest to examine how well our algorithm performs in identifying the parents of target nodes. Here, the parents of node $i$ are defined by directed edges into $i$ in the true CPDAG, since these are the parents which are identifiable with observational data alone. 

In Fig.~\ref{fig:parentrecovery}, we compare the parent recovery accuracy (PRA) of different algorithms, using the F1 score to make the comparison. The PRA F1 score is given by $\text{F1} = \frac{2TP}{2TP + FP + FN}$, where $TP$, $FP$, and $FN$ are the number of true positives, false positives, and false negatives for the estimated parent set, respectively. The F1 score is considered under two different principles for counting. The loose version of the score, used in Fig.~\ref{fig:praf1loose}, counts an estimated undirected edge between a parent and the target node in the true CPDAG as a true positive, while the strict version, used in Fig.~\ref{fig:praf1strict}, counts such an edge as a false negative. Considering again the example in Fig.~\ref{fig:illustration}, suppose we call the CPDAG of the DAG in Fig.~\ref{fig:illustrationDAG} $G'$ and denote by $G'_{NB_T}$ its subgraph over $NB_T$, where $T=\{3,8\}$. The ground truth graph, $G'_{NB_T}$, will be equivalent to the graph in Fig.~\ref{fig:cml} without the red edges. Suppose we are considering the SNL output, depicted in Fig.~\ref{fig:singlenbhd}. For $T$, the identifiable parents to be recovered are $pa_{G'_{NB_T}}(3) \cup pa_{G'_{NB_T}}(8) = \{1,2,9\}$. The SNL output correctly identifies parent edges $(1,3)$ and $(2,3)$, along with placing an undirected edge between nodes 8 and 9. Under the loose version of the score, we count the undirected edge between nodes 8 and 9 as a true positive, thus assigning the SNL output a PRA F1 score of 1. Under the strict version of the score, however, the undirected edge does not count, and the SNL output receives a score of 2/3. Distinguishing the results in two sets of scores allows us to consider how well an algorithm performs in finding possible parents (loose) as well as in providing greater specificity to the possible parent set with correctly defined invariant parent edges (strict).

When comparing algorithm performance in parent specification, we first observe that the CML algorithm is consistently better than the PC algorithm in identifying parents of the target nodes for all settings. This shows that the accumulation of errors in global learning will deteriorate performance in local neighborhoods. Second, the differences between the two plots imply that the CML algorithm correctly oriented far more directed edges than the SNL algorithm did. Consider as an example the results for $p\geq 100$ with significance level $0.01$. In the loose PRA F1 score, the SNL algorithm narrowly outperforms the CML algorithm when comparing the reported percentiles. However, in the strict version of the score, the SNL algorithm performs significantly worse than the CML algorithm and even the PC algorithm. In fact, apart from a few isolated cases, all of the PRA F1 scores for the SNL algorithm are 0 under the strict definition of the PRA F1 score. This is expected and confirms our discussion about Fig.~\ref{fig:illustration}: In the SNL algorithm, the orientation of edges in one target neighborhood has no influence on the orientation of edges in another neighborhood when those neighborhoods do not share nodes. Consequently, the SNL algorithm fails to orient as many edges as the CML algorithm, which has the advantage of inferring additional orientations using edges between neighborhoods. Table~\ref{tab:ancestors} reports the median number of between-neighborhood edges (BNE) in the CML output for each network and target size pair. An entry of N/A indicates that none of the runs for the network and target set size pair qualified for analysis. There we observe that, for larger networks and larger target sets, more between-neighborhood edges are generally identified in the sample CML output, which may help to explain its superior performance even in larger networks where we observe the deterioration of SNL performance.

\begin{table}[!t]
    \centering
        \caption{Median number of between-neighborhood edges}
    \label{tab:ancestors}
    \begin{tabular}{l|rrrr}
    \toprule
Network & $p$ & $|T|=2$ & $|T|=3$ & $|T|=4$\\
\hline
\texttt{insurance} & 27 & 1 & 1 & N/A\\
\texttt{mildew} & 35 & 1 & 1 & 0\\
\texttt{alarm} & 37 & 0 & 0 & 1\\
\texttt{barley} & 48 & 0 & 2 & 3\\
\texttt{hepar2} & 70 & 1 & 2 & 2\\
\texttt{arth150} & 107 & 1 & 1 & 3\\
\texttt{andes} & 223 & 2 & 4 & 8\\
\texttt{diabetes} & 413 & 1 & 4 & 4\\
\texttt{pigs} & 441 & 1 & 2 & 4\\
\texttt{link} & 724 & 1 & 5 & 8\\
\texttt{munin2} & 1003 & 1 & 2 & 6\\
\texttt{munin4} & 1038 & 1 & 4 & 5\\
\texttt{munin3} & 1041 & 1 & 2 & 7\\
\bottomrule
\end{tabular}
\end{table}

\section{Learning Gene Regulatory Networks}\label{sec:generegulatory}

\noindent We investigate our local methods further using a gene expression data set. 
Chu et al.~\cite{Chu2016-nz} generated single-cell RNA-sequencing (scRNA-seq) data during the early differentiation of human embryonic stem cells (hESCs) into different lineage-specific progenitor cells.
Causal discovery for a gene regulatory network in hESCs could potentially provide greater insight into critical anatomical functions and aid developments in regenerative medicine \cite{Chu2016-nz}. 

These data are available through the Gene Expression Omnibus series accession number GSE 75748. While the original data set is quite large with $n=1018$ cells and $p \approx 20,000$ genes, we use the processed data set from \cite{li2021learning}, which imputes missing values, applies a log-transformation to the expression levels, and reduces the number of genes to $p=51$ from a selection given by \cite{Chu2016-nz}.

\subsection{Data Setup}

\noindent Using a similar procedure as we did for the simulated data, we will examine different parameter combinations to compare the performance between the algorithms, and we refer to these unique parameter combinations as settings. Setting parameters for the local algorithms still consist of the significance thresholds for the Mb and skeleton recovery algorithms, $\alpha_{Mb}$ and $\alpha_{skel}$, respectively, and the maximum potential separating set size, $\ell_{max}$. The PC algorithm, as a global method, does not require $\alpha_{Mb}$, but it does require the other two setting parameters. In order to maintain uniformity, each setting is defined by parameters $\theta = (\alpha_{skel}, \ell_{max})$ for both local and global algorithms, where $\alpha_{Mb}=10^{-5}$ for all local algorithms. The Mb algorithm significance level was chosen according to the reasonable distribution of the sizes of the parent-child (P/C) sets selected by MMPC, which ranged from two to seven nodes with most P/C sets containing four or five nodes. 
All three algorithms, CML, SNL and PC, use the same settings given by $\theta \in \{10^{-5},10^{-2}\} \times \{3,5\}$.
This yields 12 total combinations of algorithms and settings. We selected target sets of cardinality $|T|\in \{2,3,4\}$ at random, with two sets of each size.

\subsection{Cross-Validation Procedure}

\noindent We randomly assign the 1,018 cells to 10 different sets, or folds, of roughly equivalent size. Our cross-validation (cv) procedure uses 10 folds, where each fold will be a held-out set used as the ``testing data'' to evaluate the model trained on the other nine folds, the ``training data.'' The results we present include the ``testing data'' likelihood scores on the ``training data'' models of all 10 folds for each parameter setting, which altogether provide a sufficient basis of comparison for the different algorithms across various settings.

For each fold of our cv procedure, we run the algorithms using the ``training data'' in order to learn the local structures around the target nodes, and then we calculate the model performance on the ``testing data.'' That is, let $\mathcal{D}$ represent the entire data set and suppose we are considering setting $i$. We partition the data $\mathcal{D}=(\mathcal{D}_{1},\mathcal{D}_{2},\ldots,\mathcal{D}_{9},\mathcal{D}_{10})$. For cv fold $j$, we define the ``training set'' as $\mathcal{D}_{\text{train};j} = \mathcal{D}_{-j} = (\mathcal{D}_{1},\mathcal{D}_{2},\ldots,\mathcal{D}_{j-1},\mathcal{D}_{j+1},\ldots,\mathcal{D}_{9},\mathcal{D}_{10})$ and the ``testing set'' is $\mathcal{D}_{\text{test};j}=\mathcal{D}_{j}$. 

Using $\mathcal{D}_{\text{train};j} = (X_1^{(-j)},X_2^{(-j)},\ldots,X_{p-1}^{(-j)},X_p^{(-j)})$, where $p$ is the total number of genes we are considering and $X_k^{(-j)}$ contains the expression levels of gene $k$ for all cells excluding those in the $j$th fold, we run CML and SNL for each set of target nodes $T=\{t_m:m=1,\ldots,|T|\}$ and obtain the estimated local structures, denoted $G^{(j)}_T$. In addition, we run the PC algorithm once for each fold and obtain the estimated structure $G^{(j)}$ over the entire node set, from which we extract the relevant subgraphs $G_T^{(j)}$ for each node set $T$. For each algorithm, we then collect the estimated parent set in $G_T^{(j)}$ for each target node, 
$\widehat{pa}_{G_T^{(j)}}(t_m)$, which we use to estimate the SEM for target node $t_m$.
For simplicity of notation, we also write $\widehat{pa}^{(j)}_{T}(t_m):=\widehat{pa}_{G_T^{(j)}}(t_m)$. 

In order to compare the results of our estimates, we use the test data log-likelihood for the held-out fold, denoted $\mathcal{D}_{\text{test};j} = (X_1^{(j)},X_2^{(j)},\ldots,X_{p-1}^{(j)},X_p^{(j)})$, where $X_k^{(j)}$ contains the expression levels of gene $k$ for cells in the $j$th fold. To compute this metric, we first estimate the SEM models for each of the target nodes $t_m \in T$ using OLS regression of $X_{t_m}^{(-j)}$ on $X^{(-j)}_{\widehat{pa}^{(j)}_{T}(t_m)}$, 
which we rewrite as $X^{(-j)}_{m;T}$ for notation brevity. This regression gives us the intercept of the model $\hat{\beta}^{(-j)}_{0m;T}$, the regression coefficients $\bm{\hat{\beta}}^{(-j)}_{m;T}$,
and the estimated error variance $(\hat{\sigma}^{(-j)}_{m;T})^2$. 
Using these estimated parameters, we can compute the test data log-likelihood, normalized by the cell number $n_j$ and the target size $|T|$, on the held-out fold as
\begin{multline*}
    \label{eqn:lltest}
    \ell\ell_{T}^{(j)} = -\frac{1}{n_{j}|T|}\sum_{m = 1}^{|T|} \bigg[\frac{n_{j}}{2}\log\left(2\pi(\hat{\sigma}^{(-j)}_{m;T})^2\right)\\ 
    + \frac{1}{2(\hat{\sigma}^{(-j)}_{m;T})^2} \left\|X_{t_m}^{(j)}- \hat{\beta}^{(-j)}_{0m;T}\mathbf{1}_{n_{j}} - X^{(j)}_{m;T}\bm{\hat{\beta}}^{(-j)}_{m;T}\right\|_2^2\bigg],
\end{multline*}
where $n_j$ is the number of cells in $\mathcal{D}_{\text{test};j}$ and $\mathbf{1}_k$ is a vector of ones with length $k$. We can now use the distribution of normalized test data log-likelihood across all cv folds to compare the different algorithms. 

\begin{figure}[!t]
    \centering
    \includegraphics[width=3in]{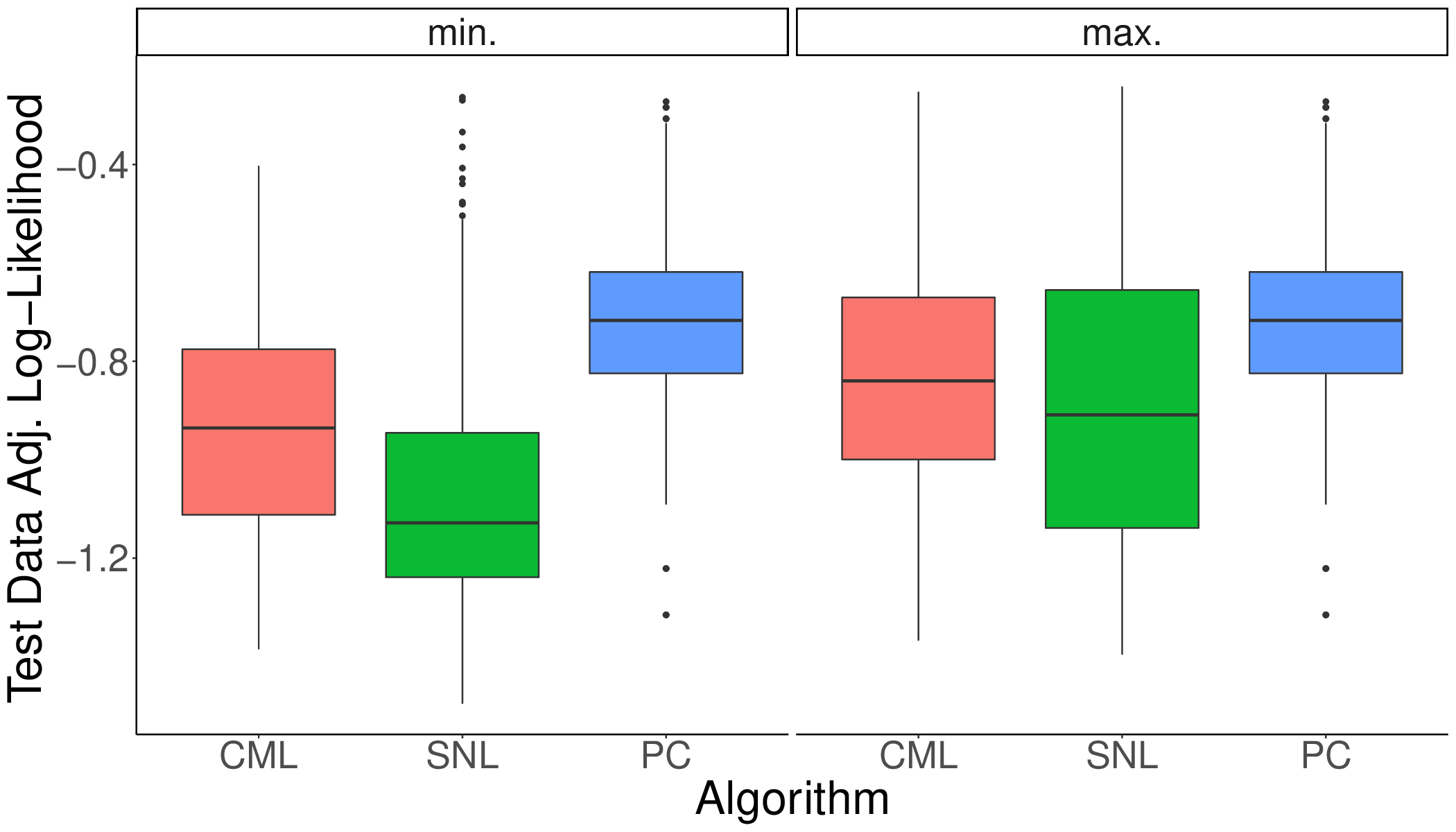}
    \caption{The distributions of test data log-likelihood scores for different algorithm and parent-identification strategy combinations. Scores are adjusted by normalizing according to the number of cells and the number of nodes in the target set for easier comparison.}
    \label{fig:rnatestll}
\end{figure}

However, this basis of comparison is somewhat limited in scope, since it automatically excludes nodes which are connected to the target nodes by an undirected edge as potential parents. Consequently, in addition to using the estimated parents to evaluate the local learning methods, we also find the maximum-sized set of jointly valid parents and recalculate the normalized test data log-likelihood with the new parent set. The maximum-sized set of jointly valid parents is generated by a modification of an output graph which converts as many undirected edges as possible to directed edges into a target node without introducing any new $v$-structures. Then, using the modified graph, we extract a new set of parents for the target nodes, which we call the maximum-sized parent sets. This helps broaden the perspective of comparison, since the additional results provide a range of possible values given the graphical output. A complete description of our procedure for finding maximum parent sets is provided in the Supplementary Material. 
We follow the same cv procedure to compare test data log-likelihood using estimated maximum-sized parent sets.


\subsection{Model Performance and Parent Recovery}

\noindent Using the cross-validation test data log-likelihood results in Fig.~\ref{fig:rnatestll}, we can identify the differences between the estimation algorithms. The plot uses two different strategies for identifying the parent sets for our model calculations, with results from each strategy provided in separate panels. The panel on the left provides the results from the estimated graph where parent sets are extracted using the directed edges, thereby representing the minimum-sized parent set (min.). The panel on the right uses the aforementioned procedure to identify the maximum-sized parent set (max.) from the estimated graph for our results. Furthermore, we can see the specific advantages of coordinated learning. One of the clearest trends in Fig.~\ref{fig:rnatestll} is the relative decline of the scores for the SNL algorithm output (green boxplot in the ``min.'' panel) in comparison to the other scores and strategies, including the SNL algorithm using the maximum parent sets. 
The CML algorithm, on the other hand, has a much smaller difference in the scores for the different parent identification strategies. The decreased variability in test data log-likelihood also highlights the value of coordinated learning, since this reduces the uncertainty in parent set identification. 
Additionally, we find that the CML algorithm improves the SNL algorithm and is fairly competitive with the global PC algorithm using the results in Fig.~\ref{fig:rnatestll}.

Moreover, Table~\ref{tab:directedcomp} provides a comparison in the number of directed edges (DE) within the same neighborhood between the SNL and CML algorithms, as well as the percentage of within-neighborhood edges which are directed (DEP), across different settings and target set sizes. The results are grouped according to settings and averaged over all cv folds and all targets sets of a given size. 
The results show a clear advantage for the CML algorithm in every case, identifying more directed edges within the target neighborhoods. It also indicates that the distributions in parent set sizes for the CML algorithm do not differ across parent identification strategies as much as they do for the SNL algorithm, due to edge orientations facilitated by the between-neighborhood edges (BNE).

\begin{table}[t]
    \centering
    \caption{Local Algorithms Directed Edges Comparison}
    \label{tab:directedcomp}
    \begin{tabular}{rrr|rr|rrr}
\toprule
 & & &   \multicolumn{2}{c}{SNL} & \multicolumn{3}{|c}{CML} \\
 $\alpha_{skel}$ & $\ell_{max}$ & $|T|$ &  DE & DEP & DE & DEP & BNE \\
\midrule
1e-05 & 3 & 2 & 1.7 & 26\% & 2.5 & 37\% & 2.0\\
1e-05 & 3 & 3 & 7.3 & 54\% & 9.9 & 61\% & 0.9\\
1e-05 & 3 & 4 & 8.5 & 46\% & 10.8 & 53\% & 1.8\\ \hline
1e-05 & 5 & 2 & 0.0 & 0\% & 1.2 & 20\% & 2.0\\
1e-05 & 5 & 3 & 1.0 & 10\% & 3.1 & 26\% & 0.9\\
1e-05 & 5 & 4 & 1.4 & 12\% & 5.0 & 35\% & 1.7\\ \hline
1e-02 & 3 & 2 & 3.0 & 28\% & 4.8 & 43\% & 3.7\\
1e-02 & 3 & 3 & 8.0 & 48\% & 11.0 & 57\% & 3.0\\
1e-02 & 3 & 4 & 9.4 & 38\% & 13.0 & 51\% & 3.9\\ \hline
1e-02 & 5 & 2 & 2.5 & 31\% & 4.0 & 44\% & 3.5\\
1e-02 & 5 & 3 & 5.2 & 35\% & 9.6 & 56\% & 2.9\\
1e-02 & 5 & 4 & 6.4 & 30\% & 11.2 & 50\% & 2.8\\
\bottomrule
\end{tabular}
\end{table}

Finally, we take a closer look at the complexity comparison between the various algorithms. The two modes in the distributions of the number of tests and of the runtimes in Fig.~\ref{fig:sccomplexityandruntime} correspond to the two significance levels we use for our experiment: the fewer tests (faster times) for $\alpha_{Mb}=10^{-5}$ and the greater tests (slower times) for $\alpha_{Mb}=10^{-2}$. In Fig.~\ref{fig:sctestplot}, we observe that, for target sets with 2 or 3 nodes, we can confidently expect the number of tests for the local algorithms to be fewer than those of the global algorithm. We do not have the same results in Fig.~\ref{fig:sctimeplot}, but we do see that the time distributions of the local algorithms mostly overlap with the bottom of the distribution for the global PC algorithm, which corresponds to the smaller of the significance levels ($\alpha_{skel}=10^{-5}$) used for the skeleton CI tests. For these results, Fig.~\ref{fig:sctimeplot} reports a fair comparison in runtime between local and global algorithms, where we manually calculated the number of tests for the PC algorithm in a separate run of the same procedure. In both plots, we see a clear increase in complexity as the size of the target set increases for local algorithms. Since this is a relatively low-dimensional data set, larger target sets will include most of the nodes in the network as their first- and second-order neighbors, which is a less than ideal situation for local algorithms. These results cohere with previous complexity discussions, especially since this is a smaller data set.

\begin{figure}[!t]
    \centering
    \subfloat[]{
    \includegraphics[width=3in]{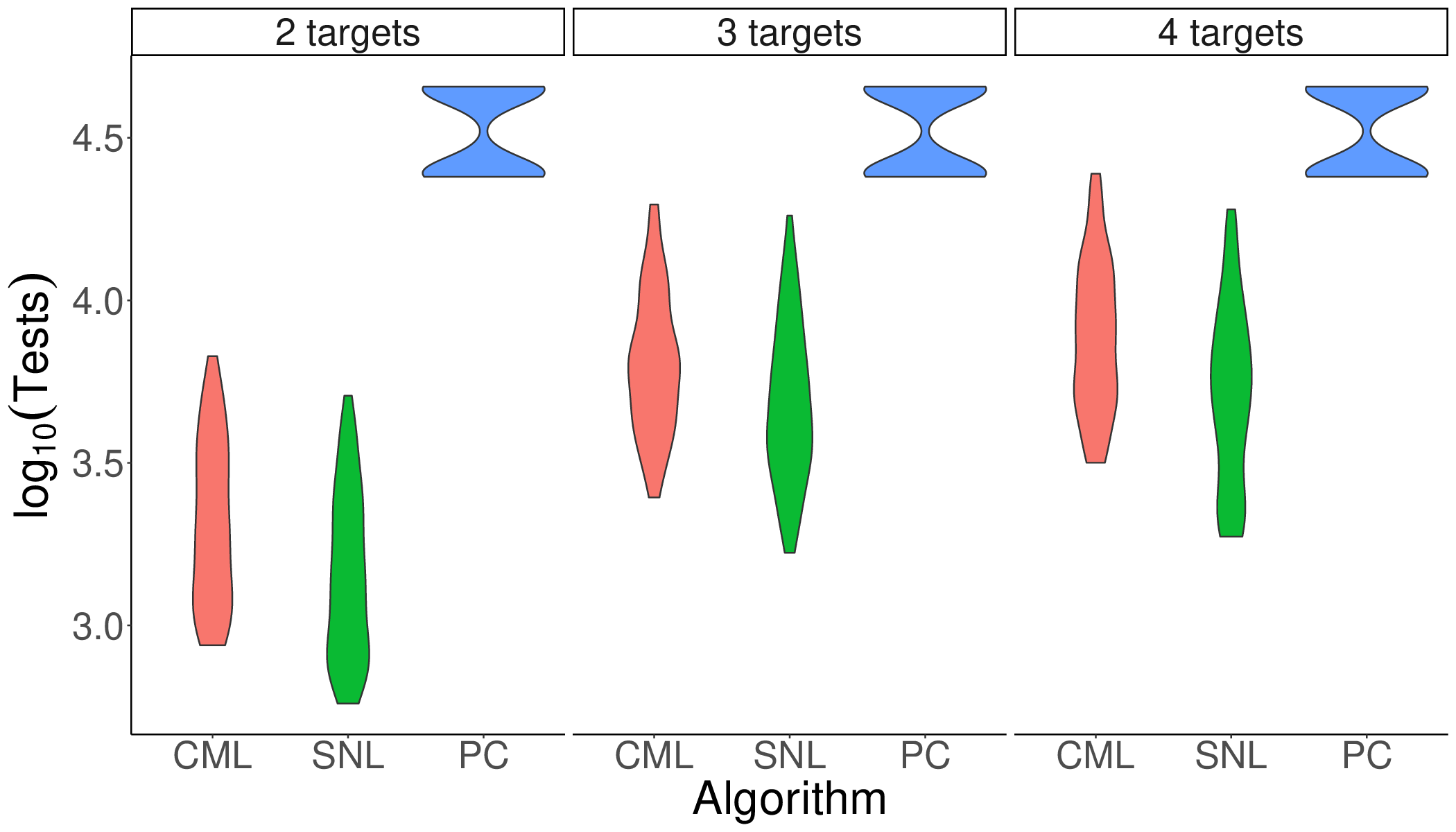}%
    \label{fig:sctestplot}
    }
    \hfil
    \subfloat[]{
    \includegraphics[width=3in]{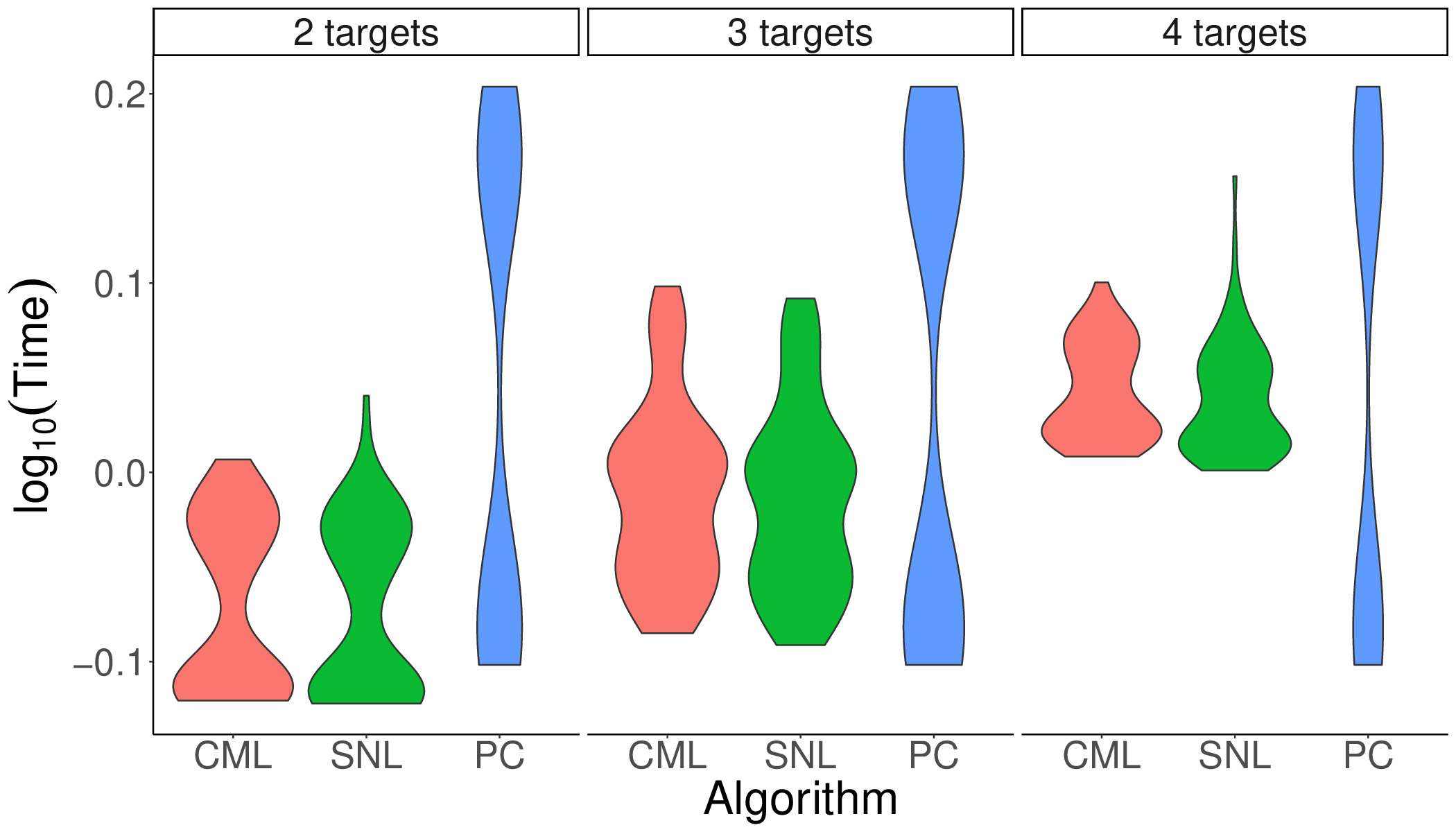}%
    \label{fig:sctimeplot}
    }
    \caption{Comparisons between the global and local algorithms with respect to complexity. (a) The distributions of the number of CI tests and (b) runtime for varying target set sizes on a log scale.}
\label{fig:sccomplexityandruntime}
\end{figure}

The results from this data set provide support for the claims we have presented in this work. Coordinated learning provides clear benefits in comparison to learning within individual neighborhoods, especially in estimation precision for modeling. Additionally, assuming a sparse network with target neighborhoods of a significantly smaller size than the dimension of the underlying DAG, CML is competitive with a global algorithm such as the PC.

\section{Conclusion}
\label{sec:conclusion}

\noindent For estimating the causal structure in local neighborhoods and facilitating the estimation of causal effects of specified target nodes, Coordinated Multi-Neighborhood Learning is a sound algorithm with demonstrable empirical benefits. Compared to existing methods, this algorithm is more efficient and scalable while maintaining a degree of accuracy comparable to or better than global methods. However, as we observed with the gene expression data set, the extent of these efficiency gains depends on multiple factors, such as the selected target sets and the size of the neighborhoods. Though in this work we only conducted empirical analysis on Gaussian data, the method and theoretical results developed in this paper are more broadly applicable to different kinds of data and DAG models. A future direction of research is to observe the performance of our algorithm with simulated and real-world data from different distributions. Along similar lines, another potentially fruitful research direction would consider modifications to our algorithm for the use of experimental data with interventions. Additionally, because this algorithm depends on Mb recovery, it is constrained in its effectiveness by the performance of the pre-processing method. We may also realize greater efficiency by better integrating the Mb estimation algorithm and CML to avoid repeating the same CI tests for skeleton recovery. Finally, the current selection of potential separating sets during the skeleton stage is too conservative. In the future, we may restrict our attention to only subsets of one of the adjacency sets at a time rather than to subsets of the union over both adjacency sets, which should improve the efficiency and speed of the algorithm.

\section{Acknowledgements}

\noindent This work was supported by NSF grants DMS-1952929 and DMS-2305631. This work used computational and storage services associated with the Hoffman2 Shared Cluster provided by UCLA Institute for Digital Research and Education’s Research Technology Group.

\appendix

\section{Proofs of Theoretical Results}

\begin{proof}[Proof of Lemma~\ref{lm:mag}]
As a subgraph of DAG $\calG$, the graph $\calG_N$ is a DAG as well and it follows that $\calG_N$ is ancestral because it has no directed cycles. For the additional edges in $\calG^*_N$, it follows from their construction that no directed or almost directed cycles will be introduced. For any $(i,j) \in B$ such that there is an inducing path relative to $L = V - N$, if we set $i \rightarrow j$, then it follows that $j \notin an_\calG(i)$ and $j \notin an_{\calG^*_N}(i)$. If we set $i \leftrightarrow j$, then $i \notin an_\calG(j)$ and $j \notin an_\calG(i)$ by construction, which implies $i \notin an_{\calG^*_N}(j)$ and $j \notin an_{\calG^*_N}(i)$. Therefore, $\calG^*_N$ is ancestral since it has neither directed nor almost directed cycles. It is proved in \cite{richardsonAncestralGraphMarkov2002} that DAGs are maximal ancestral graphs, and therefore $\calG_N$ is a maximal ancestral graph. After we add directed edges $i \rightarrow j$ for $(i,j) \in B$, we still have a DAG and thus preserve maximality. 
Furthermore, we would retain maximality after bidirected edges are added between nodes in distinct neighborhoods. We prove the last assertion by contradiction. Assume there is an inducing path $p_I = \langle \alpha,\beta,\gamma,\ldots,\epsilon,\omega\rangle$ in $\calG_N^*$ with non-adjacent endpoints $\alpha,\omega$. The orientation of the edges on $p_I$ must be $\alpha \astarrow \beta \leftrightarrow \gamma \leftrightarrow \cdots \leftrightarrow \epsilon \arrowast \omega$, where $\ast$ is a wildcard which can represent either a tail or an arrowhead. At least one of the $\ast$'s must be an arrowhead because each of $\beta$ and $\epsilon$ is an ancestor of either $\alpha$ or $\omega$ by the definition of inducing path. Without loss of generality, let us assume $\alpha\leftrightarrow\beta$, which implies that $\alpha$ and $\beta$ are not in the same neighborhood because bidirected edges only exist between distinct neighborhoods in $\calG_N^*$. It is easy to see that this path $p_I$ corresponds to an inducing path relative to $L$ in the original DAG $\calG$, and by Assumption~\ref{asp:inp}, $\alpha$ and $\omega$ must be in different target neighborhoods. Therefore, by construction $(\alpha,\omega)$ is an edge in $\calG^*_N$. This contradicts the assumption that $\alpha$ and $\omega$
are non-adjacent endpoints of the path $p_I$. 
\end{proof}

\begin{proof}[Proof of Theorem~\ref{thm:pop}]
We begin by showing that we recover the skeleton of $\calG^*_N$ after the skeleton recovery stage of our algorithm. Since the distribution $P(X_1,\ldots,X_p)$ is faithful to $\calG$, for $i,j\in N$ and $S\subseteq N$, conditional independence of $X_i$ and $X_j$ given $X_S$ is equivalent to $m$-separation of nodes $i$ and $j$ given set $S$. 
Therefore, after line~\ref{lst:line:endglobalskel}, 
the edge set $E$ corresponds to the skeleton of the true MAG over $N$, which is a supergraph of the skeleton of $\calG^*_N$. For any $t\in T$ and $(i,j) \in NB_t$, subsets of $N_i^1$ and of $N_j^1$ will be sufficient to remove edges between non-adjacent $i,j$ in $\calG$. Then, after having used second-order neighbors within each neighborhood, we obtain the skeleton of $\calG^*_N$ after line~\ref{lst:line:endtargetnbhdskel}.

While the correctness of the FCI rules have been shown by \cite{zhang2008}, it remains to show that our use of the rules in the CML algorithm is valid. The rules only depend on the skeleton and whether a node $\gamma\in N$ is in a separating set $S_{ij}$ ($\mathcal{R}_0$ and $\mathcal{R}_4$). For any separating set $S_{ij}$ found in the skeleton recovery stage, let $S'_{ij}=S_{ij} \cap N$. 
In the application of the FCI rules involving the separating set $S_{ij}$, using $S'_{ij}$ instead will lead to the same orientation result since $\gamma \in S_{ij}$ if and only if $\gamma \in S'_{ij}$ for any $\gamma \in N$. For an edge $(i,j)$ deleted in the first phase of skeleton learning, $S'_{ij}=S_{ij}$. If $(i,j)$ is deleted in the second phase given the separating set $S_{ij}$, then $S'_{ij}$ is a separating set for $i$ and $j$ in $\calG^*_N$. This is because a path blocked by second-order neighbors $Z=S_{ij}\setminus S'_{ij}$ becomes disconnected in $\calG^*_N$ after $Z$ is removed. Thus, $\{S_{ij}'\}$ consists of separating sets for all pairs of non-adjacent nodes $\calG^*_N$, which guarantees sound and complete orientation of $[\calG^*_N]$ by the FCI rules. This completes the proof.
\end{proof}

Proofs of Corollary~\ref{cor:Gauconsistency} and Theorem~\ref{thm:pagconsistency} are provided in the Supplementary Material.

\bibliographystyle{IEEEtran}
\bibliography{cml}

\end{document}